\documentclass[transaction]{IEEEtran}
\usepackage{amsfonts}
\usepackage{amssymb}
\usepackage{hyperref}
\usepackage{graphicx}
\usepackage{enumerate}
\usepackage{amsmath}
\usepackage{color}
\usepackage{amsthm}
\usepackage{verbatim}
\usepackage{amsmath}
\usepackage{algorithm}
\usepackage{cite}
\usepackage{algpseudocode}
\usepackage{subfig}
\usepackage{graphicx}
\usepackage{float}
\usepackage{autobreak}
\newcommand{\bp}{\begin{proof} \small }
\newcommand{\ep}{\end{proof} \normalsize}
\newcommand{\epx}{\end{proof} \small}
\newcommand{\bpa}{\begin{proofappx} \footnotesize }
\newcommand{\epa}{\end{proofappx} \small }
\newtheorem{theorem}{Theorem}
\newtheorem{proposition}{Proposition}
\newtheorem{corollary}{Corollary}
\newtheorem{lemma}{Lemma}
\newtheorem{assumption}{Assumption}

\newtheorem*{theorem*}{Theorem}
\newtheorem*{proposition*}{Proposition}
\newtheorem*{corollary*}{Corollary}
\newtheorem*{lemma*}{Lemma}
\newtheorem*{assumption*}{Assumption}
\newtheorem*{definition*}{Definition}
\newtheorem*{claim*}{Claim}

\newcommand{\be}{\begin{equation}}
\newcommand{\ee}{\end{equation}}
\newcommand{\bs}{\begin{subequations}}
\newcommand{\es}{\end{subequations}}
\newcommand{\bq}{\begin{eqnarray}}
\newcommand{\eq}{\end{eqnarray}}
\newcommand{\bqn}{\begin{eqnarray*}}
\newcommand{\eqn}{\end{eqnarray*}}

\newcommand{\ba}{\left[ \begin{array}}
\newcommand{\ea}{\\ \end{array} \right]}
\newcommand{\ben}{\begin{enumerate}}
\newcommand{\een}{\end{enumerate}}

\def\u{{\boldsymbol{u}}}
\def\v{{\boldsymbol{v}}}

\def\x{{\boldsymbol{x}}}
\def\y{{\boldsymbol{y}}}

\def\real{{\mathchoice%
{\hbox{\rm\setbox1=\hbox{I}\copy1\kern-.45\wd1 R}}
{\hbox{\rm\setbox1=\hbox{I}\copy1\kern-.45\wd1 R}}
{\hbox{\scriptsize\rm\setbox1=\hbox{I}\copy1\kern-.45\wd1 R}}
{\hbox{\scriptsize\rm\setbox1=\hbox{I}\copy1\kern-.45\wd1 R}}}}

\def\Zint{{\mathchoice{\setbox1=\hbox{\sf Z}\copy1\kern-.75\wd1\box1}
{\setbox1=\hbox{\sf Z}\copy1\kern-.75\wd1\box1}
{\setbox1=\hbox{\scriptsize\sf Z}\copy1\kern-.75\wd1\box1}
{\setbox1=\hbox{\scriptsize\sf Z}\copy1\kern-.75\wd1\box1}}}
\newcommand{\complex}{ \hbox{\rm C\kern-0.45em\rule[.07em]{.02em}{.58em}%
\kern 0.43em}}

\allowdisplaybreaks

\begin{document}
%
\title{On the Local Cache Update Rules \\ in Streaming Federated Learning}
%
%
%

\author{Heqiang Wang, Jieming Bian, Jie Xu
\thanks{Heqiang Wang, Jieming Bian and Jie Xu are with the Department of Electrical and Computer Engineering, University of Miami, Coral Gables, FL 33146, USA. }

}

\maketitle

\begin{abstract}
In this study, we address the emerging field of Streaming Federated Learning (SFL) and propose local cache update rules to manage dynamic data distributions and limited cache capacity. Traditional federated learning relies on fixed data sets, whereas in SFL, data is streamed, and its distribution changes over time, leading to discrepancies between the local training dataset and long-term distribution. To mitigate this problem, we propose three local cache update rules - First-In-First-Out (FIFO), Static Ratio Selective Replacement (SRSR), and Dynamic Ratio Selective Replacement (DRSR) - that update the local cache of each client while considering the limited cache capacity. Furthermore, we derive a convergence bound for our proposed SFL algorithm as a function of the distribution discrepancy between the long-term data distribution and the client's local training dataset. We then evaluate our proposed algorithm on two datasets: a network traffic classification dataset and an image classification dataset. Our experimental results demonstrate that our proposed local cache update rules significantly reduce the distribution discrepancy and outperform the baseline methods. Our study advances the field of SFL and provides practical cache management solutions in federated learning.
\end{abstract}


%
\IEEEpeerreviewmaketitle

\section{Introduction}
Federated learning (FL) is a distributed machine learning paradigm that enables a set of clients with decentralized data to collaborate and learn a shared model under the coordination of a centralized server. In FL, data is stored on edge devices in a distributed manner, which reduces the amount of data that needs to be uploaded and decreases the risk of user privacy leakage. While FL has gained popularity in the field of distributed deep learning, most research on FL has been conducted under ideal conditions and has not fully accounted for real-world constraints and features. Given that the client in FL is typically an edge device, we highlight two features that are more aligned with reality. The first feature, called \textbf{Streaming Data}, acknowledges that clients often consist of edge devices that continually receive and record data samples on-the-fly. Therefore, FL must operate on dynamic datasets that are built on incoming streaming data, rather than static ones. The second feature, called \textbf{Limited Storage}, recognizes that edge devices such as network routers and IoT devices have limited storage space allocated for each service and application. As a result, only a restricted amount of space can be reserved for FL training without compromising the quality of other services. This paper aims to address the lack of consideration for these two real-world features in current FL research.

To address the problem presented above, we investigate a new FL problem, called Streaming Federated Learning (SFL), where the local models of clients are trained based on dynamic datasets rather than static ones. SFL involves three different types of data distributions. The first one is the long-term (underlying) label distribution, which pertains to the data distribution of the client following a prolonged period of streaming data reception. This distribution cannot be anticipated during the training process and, due to storage capacity constraints, it is unfeasible to obtain an accurate long-term distribution by recording the whole data stream. The second type is the short-term (empirical) label distribution, which corresponds to the distribution of the client's currently received data. Short-term distributions are noisy and may vary over time, and they may not necessarily approximate the long-term distribution. The discrepancy between the long-term distribution and short-term distributions is illustrated in Fig.~\ref{lsc}. The third type is the cached label distribution, which is the distribution of the dataset currently stored in the client and is governed by the local cache update rule. The aforementioned three distributions suggest that the primary challenge of SFL is the discrepancy between the a priori unknown long-term distribution and the distribution of cached data for training, as the training data is continually gathered from the stream. As a result, a proper local dataset update rule is essential to produce a cached distribution based on the short-term distributions so that it captures the long-term distribution as accurately as possible, thereby enhancing the learning performance. Our main contributions are summarized as follows:
\begin{enumerate}
    \item  We formulate the SFL problem and propose new FL algorithms for SFL. Unlike conventional FL, training in SFL must be conducted on a dynamic dataset based on streaming data rather than a static dataset. The clients in SFL also only have limited storage capacity, making storing all incoming data impossible.
    \item We propose and investigate three different local dataset update rules and theoretically analyze the discrepancy between the resulting cached distributions and the long-term distribution. Based on this discrepancy analysis, we further prove a convergence bound of our proposed SFL algorithm.
    \item We apply SFL to address a practical problem, namely online training of network traffic classifiers. Our experiments, which use both a network traffic classification dataset and the FMNIST dataset, demonstrate that our proposed update rules outperform benchmarks in the SFL framework. 
\end{enumerate}

The rest of this paper is organized as follows. In Section II, we discuss related works on FL and network traffic classification. Section III presents the system model and formulates the SFL problem. In Section IV, we introduce the SFL workflow, propose three local dataset update rules, and analyze discrepancies. Section V presents the convergence analysis of SFL on non-i.i.d. data and a non-convex function. The experimental results of SFL are presented in Section VI. Finally, Section VII concludes the paper. 

\section{Related Work}
In recent years, FL has emerged as a promising framework for decentralized deep learning. Several works, such as \cite{li2020federated, lim2020federated, yang2019federated, wahab2021federated}, have provided a comprehensive introduction to FL and its research problems. Among the various challenges in FL, the convergence analysis of FedAvg and its variants stands out as particularly crucial. Early works focused on FL under the assumptions of i.i.d. datasets and full client participation, as demonstrated in \cite{stich2018local, yu2019parallel, wang2018cooperative}. Most of the theoretical works suggest that convergence occurs linearly given a sufficiently large number of learning rounds. However, this assumption may not always hold in real-world FL scenarios, leading to an increasing number of works \cite{karimireddy2020scaffold, li2019convergence, yang2022anarchic, jhunjhunwala2022fedvarp, yang2021achieving} investigating convergence proofs of FedAvg and its variants under non-i.i.d. datasets. Although the proposed SFL differs from conventional FL, our proof is primarily inspired by \cite{yang2021achieving}, which is based on non-convex functions and non-i.i.d. datasets.

Recently, there has been a growing interest in exploring non-stationary and continually evolving local datasets that are known as concept drift problems \cite{lu2018learning, hoens2012learning}. Some researchers have begun to investigate the FL with concept drift \cite{chen2021asynchronous, canonaco2021adaptive, jothimurugesan2022federated, casado2022concept}. To address the concept drift problem under the FL setting, various approaches have been proposed, such as adjusting learning rates \cite{canonaco2021adaptive}, incorporating regularization terms \cite{chen2021asynchronous, casado2022concept}, or training multiple models separately \cite{jothimurugesan2022federated}. However, although SFL considered in our paper has some connection to the concept drift problem, there is a fundamental difference. In the SFL problem, the global objective function remains constant depending on the constant albeit unknown long-term label distribution. However, in the concept drift problem, the underlying distribution changes, leading to a change in the global objective function. Apart from the concept drift problems, two works \cite{jin2021budget} and \cite{gong2022ode} consider a similar streaming data structure that is more relevant to our proposed SFL. In \cite{jin2021budget}, the authors propose an online approach to control local model updates on streaming data and global model aggregations of FL, with the aim of preventing training load congestion after the preparation of the entire training data and ensuring that model training is spread out with the arrival of streaming data. Meanwhile, our study focuses on exploring the discrepancy between long-term label distribution and cached label distribution that arises from FL with streaming data and limited storage. The authors in \cite{gong2022ode} introduce an online data selection framework for FL with streaming data. They aim to allow the server to exert control in a way that gradually regulates the data distribution of all clients to approach an i.i.d. distribution by facilitating additional information exchange between the server and the clients. However, in our work, we consider a more practical scenario where the clients themselves are responsible for data selection.

In terms of application, we applied SFL to online training of network traffic classifiers. Network traffic classification involves categorizing network traffic data into different types or classes based on certain characteristics of the data. There are two main categories of methods for performing network traffic classification: traditional methods and machine learning-based methods. Traditional methods mainly rely on port \cite{moore2005toward} or payload \cite{finsterbusch2013survey} traffic classification approaches. However, these methods can fail when faced with port translation or encrypted network packets \cite{taylor2017robust}. With the growing popularity of deep learning, some recent approaches have utilized neural networks for traffic classification. For instance, in \cite{liu2019fs}, the authors proposed FS-Net based on recurrent neural networks and autoencoder for traffic classification and packet feature mining. Another approach, described in \cite{zhang2020autonomous}, uses a DL-based autonomous learning framework for traffic classification, which can also handle unknown classes. Nonetheless, these approaches typically rely on centralized deep learning models, which may not be the optimal choice for distributed scenarios involving edge devices such as routers. Some recent works \cite{mun2020internet, peng2021federated} have suggested using the FL approach to address the issue of traffic classification. However, their proposed solutions do not take into account the gradual arrival of data in network traffic problems or the limited storage capacity of edge devices like routers.

\section{Problem Formulation}
\subsection{System Model}
Let us consider a network consisting of one server and $K$ clients. Unlike conventional FL frameworks that use static datasets, every client in the considered system gradually acquires data from its online data source, and each of these online data sources has a long-term (underlying) label distribution. To facilitate exposition, we discrete time into periods (each of which corresponds to a learning round as we will define shortly) and assume that each client $k \in \{1,..., K\}$ receives a set $\mathcal{S}^k_t$ of $B_s$ labeled data samples from its online data source in each period $t$. Each client $k$ has a finite cache $\mathcal{L}^k$ of size $B > B_s$. For analytical simplicity, we assume that $B$ is a multiple of $B_s$ and denote $M = \frac{B}{B_s} \in \mathbb{Z}_+$. Because the client cache is limited, not all labeled data samples can be stored and used for learning at the same time. 

\begin{figure}[htbp]
\centering
\includegraphics[width=5cm]{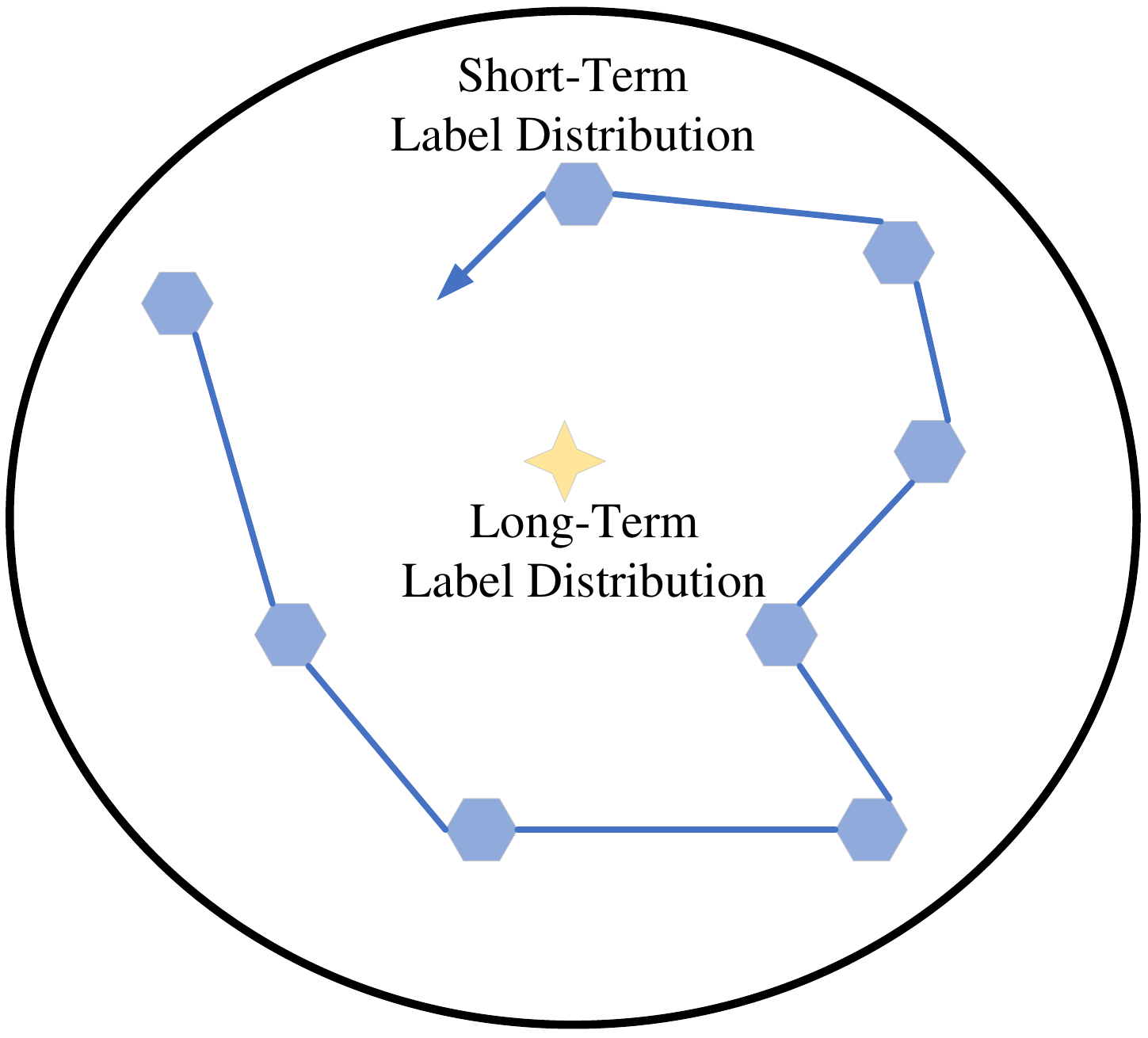}
\caption{\label{lsc} The trajectory of the long-term label distribution and the short-term label distribution. }
\end{figure}

Each client $k$ has a long-term label distribution $\pi^k = [\pi^{k,1}~\pi^{k,2}~...~\pi^{k,R}]$, where $R$ is the total number of label classes and $\pi^{k,r}$ represents the probability that class $r$ appears in client $k$, which is unknown by the client beforehand. However, the short-term (empirical) label distribution can be different from the long-term label distribution and non-stationary over time as shown in Fig.~\ref{lsc}. For example, in network traffic classification, productivity applications may take up a large portion of network traffic in the daytime while entertainment applications are more popular at night. As a result, the application label distribution of $\mathcal{S}^k_t$ in one period is noisy and biased due to not only the finite number of instances but also the non-stationary application usage patterns. Furthermore, the short-term label distribution often does not change abruptly but exhibits temporal correlations. In other words, the application label distributions in the received labeled dataset may be similar in adjacent periods. Let $n^{k,r}_t$ be the number of instances with label $r$ in $\mathcal{S}^k_t$ and we denote $\u^k_t = [u^{k,1}_t~u^{k,2}_t~...~u^{k,R}_t]$ as the short-term label distribution of $\mathcal{S}^k_t$ where $u^{k,r}_t = n^{k,r}_t/B_s$. We make the following assumptions on $\u^k_t$. 
\begin{assumption}[Limited Temporal Correlation]\label{assm1}
There exists an integer $\Gamma > 0$ such that (1) for any $\tau \leq \Gamma$ we have $0 < \max_{k, r, t}\mathbb{E}[(u^{k,r}_t - \pi^{k,r})(u^{k,r}_{t-\tau} - \pi^{k,r})] \leq \delta^2$ for some constant $\delta^2$; (2) for any $\tau > \Gamma$ we have $\mathbb{E}[(u^{k,r}_t - \pi^{k,r})(u^{k,r}_{t-\tau} - \pi^{k,r})] = 0$. 
\end{assumption}
Assumption \ref{assm1} states that the temporal correlation of the label distribution is confined in a neighborhood of $\Gamma$ periods. For analytical simplicity, we assume the same $\delta^2$ for any $\tau \leq \Gamma$ but practically it makes sense that $\delta^2$ is larger for smaller $\tau$ since closer periods exhibit stronger correlation. This generalization is straightforward in our framework. 

Because the client has a finite cache, we also define the cached label distribution at client $k$ in period $t$, denoted by $\v^k_t = [v^{k, 1}_t~v^{k,2}_t~...~v^{k, R}_t]$, as the label distribution of data currently in the cache. The cached label distribution is a joint result of both the short-term distribution and the local cache update rule. 

To better understand these concepts, consider the network traffic classification problem. Each local area network (LAN) $k$ connects to the network via a router/access point $k$, which monitors the application usage in the LAN. These routers act as the client in FL. Suppose there are a total number of $R$ possible applications and network traffic classification aims to identify the application $y \in \{1, ..., R\}$ based on the data packet feature $x$. In our problem, we consider that labeled data packets continuously arrive at the routers depending on the application usage pattern in the LAN for training the DL-based network traffic classifier. The labeled data packets may be manually labeled with delay and the number is kept small relative to the total data traffic in order to reduce the labeling overhead and complexity. It is important to note that the network traffic classifier problem represents only one instance of the broader SFL problem. In utilizing the network traffic classifier problem to illustrate SFL, our aim is simply to aid the reader's comprehension of the problem. 

\subsection{Learning Objective}
Our goal is to train a machine learning model using the limited number of labeled data samples received by the different clients. Without loss of generality, we assume that the data arrival rate to all clients is the same. Therefore, the long-term label distribution of the overall network is simply the average of that of each client, i.e., $\pi = \frac{1}{K} \sum_{k = 1}^K \pi^k$. We define the loss function as $f(w) = \mathbb{E}_{\xi \sim \pi}F(w;\xi)$ where $F(w;\xi)$ is the objective function with data sample/s $\xi$, $\xi$ represents the sample/s drawn from the long-term label distribution, and the loss function can further be decomposed into a weighted sum of local loss functions as follows
\begin{align}
f(w) = \frac{1}{K} \sum_{k=1}^K f^k(w) = \frac{1}{K} \sum_{k=1}^K \mathbb{E}_{\xi \sim \pi^k} F^k(w; \xi)
\end{align}
where $f^k(w) = \mathbb{E}_{\xi \in \pi^k} F^k(w; \xi)$ is the local loss function of client $k$. Thus, training the machine learning model is equivalent to solving for the optimal parameter $w$ that minimizes the loss function, i.e., $\min_w f(w)$. 

Because of the distributed nature of the network, it is impractical to send all the labeled data samples to a central location to train the model. Privacy concerns can also be another reason that forbids clients from directly exchanging data with each other. In this paper, we take the FL approach to train the machine learning model in a distributed manner assisted by a parameter server, where clients train local models based on their local data and periodically exchange the local models with a parameter server to derive the global model. However, compared to conventional FL systems where local models are trained on static local datasets, the online machine learning model must be trained on time-varying dynamic data. As the labeled data samples are received gradually over time at the clients, the clients do not have access to the long-term label distribution at the beginning but must continuously update their finite local cache for the incoming training instances. The cached label distribution in the local cache may diverge from the long-term label distribution because of the short-term non-stationarity, thereby degrading the FL performance. 

In the next sections, we introduce the SFL architecture for online machine learning model training and investigate how different local cache updating rules affect learning performance.

\section{SFL architecture and local cache update rules}
\subsection{SFL Architecture}
In the proposed SFL system, learning is organized into a series of iterative learning rounds. As previously mentioned, one period corresponds to a learning round. Each learning round $t$ comprises the following four steps. 
\begin{enumerate}
    \item \textbf{Global Model Download}. Each client $k$ downloads the current global model $w_t$ from the parameter server. 
    \item \textbf{Local Model Update}. Each client $k$ uses $w_t$ as the initial model to train a new local model $w^k_{t+1}$ based on the current training data samples in its local cache $\mathcal{L}^k$. Because the local cache is finite and usually small, we consider local training performs $E$ steps of full-batch gradient descent (GD). Specifically, the local model is updated as
    \begin{align}
        & w^k_{t, 0} = w_t\\
        & w^k_{t, \tau+1} = w^k_{t, \tau} - \eta_L g^k_{t,\tau}, \forall \tau = 1, ..., E\\
        & w^k_{t+1} = w^k_{t, E}
    \end{align}
    where $g^k_{t, \tau} = \nabla F^k(w^k_{t, \tau}; \mathcal{L}^k_t)$ is the gradient computed on the local dataset currently stored in the local cache $\mathcal{L}^k_t$, and $\eta_L$ is the local learning rate. Note that because of the short-term non-stationarity and finite cache space, $\nabla F^k(w; \mathcal{L}^k_t)  = \mathbb{E}_{\xi \sim \pi^k} F(w; \xi)$ does not hold.
    \item \textbf{Local Model Upload}. Clients then upload their local model updates to the server. Typically, instead of uploading the local model $w^k_{t+1}$, client $k$ may upload only the local model update $\Delta^k_t$, which is defined as the total model difference as follows:
    \begin{align}
        \Delta^k_t = \frac{1}{\eta_L}(w^k_{t, E} - w^k_{t, 0}) = -\sum_{\tau=0}^{E-1}g^k_{t, \tau}
    \end{align}
    \item \textbf{Global Model Update}. The server updates the global model by using the aggregated local model updates from the clients:
    \begin{align}
        w_{t+1} = w_t + \eta\eta_L\Delta_t, \text{where}~\Delta_t = \frac{1}{K}\sum_{k=1}^K\Delta^k_t
    \end{align}
    where $\eta$ is the global learning rate. 
\end{enumerate}

\begin{figure}[htbp]
\centering
\includegraphics[width=9cm]{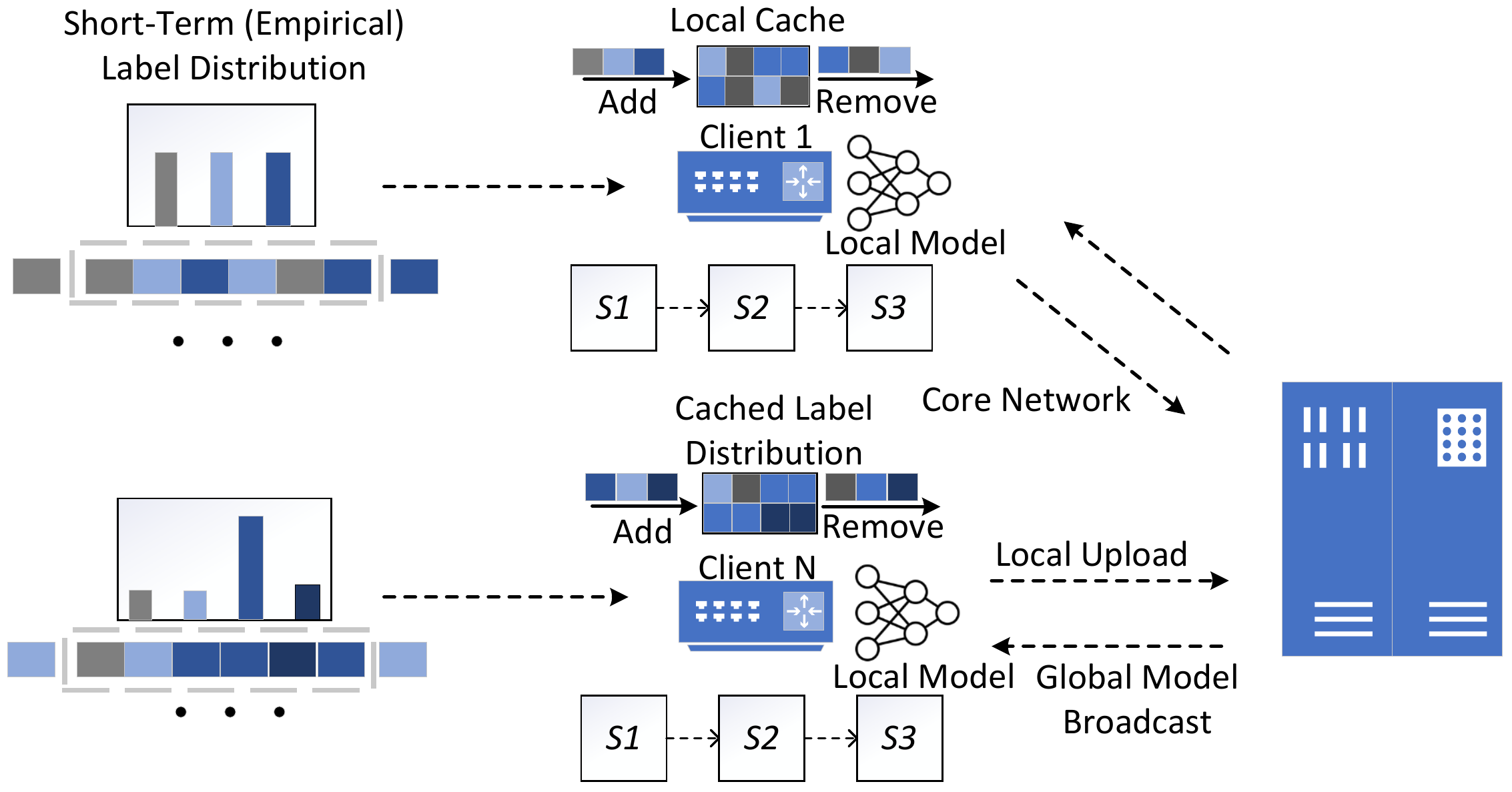}
\caption{\label{sfl} The SFL framework. There are three main stages in SFL: $S_1$ streaming data arrival. $S_2$ local cache updating. $S_3$ local model updating.}
\end{figure}

\subsection{Local Cache Update}
The key difference between conventional FL and SFL is how the local model update is performed, specifically, what data the local model is trained on. In conventional FL, the local model is trained on a static local dataset (using either all data or sampled data) whereas in SFL, the local dataset must be continuously updated as new data is received and old data is removed. Therefore, the local cache update rule will affect what data is used for training local models and consequently the global learning performance. 

We illustrate the streaming data arrival and local cache updating in Fig.~\ref{sfl}. Between two consecutive local model updates, new labeled data is received by the clients. In particular, client $k$ receives a labeled dataset $\mathcal{S}^k_t$ by the local cache update step in round $t$. Then client $k$ updates the local cache $\mathcal{L}^k_t$ using the new data $\mathcal{S}^k_t$ and the existing data in the local cache according to some update rule $\Phi$ as follows
\begin{align}
    \mathcal{L}^k_t \gets \Phi(\mathcal{L}^k_{t-1}, \mathcal{S}^k_t)
\end{align}
The updated local cache is then used for local model training at client $k$. Next, we introduce several local cache update rules. 

\subsubsection{First-In First-Out (FIFO)}
A straightforward local cache update rule is FIFO, which is also used as a baseline for many other caching systems. Specifically, the FIFO update rule uses queuing logic to remove the oldest data  so that a newly received data instance can be added. In our problem, client $k$ simply removes the $B_s$ oldest labeled data instances, denoted by $\mathcal{H}^k_{t-1}$, to make room for the new $B_s$ labeled data instances in $\mathcal{S}^k_t$. Mathematically, 
\begin{align}
    \mathcal{L}^k_t \gets \mathcal{L}^k_{t-1}\backslash \mathcal{H}^k_{t-1} \cup \mathcal{S}^k_t
\end{align}
The cached label distribution changes as a result of the updated local cache as follows:
\begin{align}
    v^{k,r}_t = \frac{n_r(\mathcal{L}^k_{t-1}) - n_r(\mathcal{H}^k_{t-1}) + n_r(\mathcal{S}^k_t)}{B}, \forall r = 1, ..., R
\end{align}
where $n_r(\mathcal{X})$ is the number of data instances with label $r$ in a set $\mathcal{X}$. We characterize the discrepancy between the cached distribution and the long-term distribution below. 

\begin{proposition}\label{prop1}
The discrepancy between the cached label distribution and the long-term label distribution by using FIFO is bounded as follows,
\begin{align}
    \mathbb{E}[(v^{k,r}_t - \pi^{k,r})^2] \leq \frac{1}{M} (\min\{2\Gamma+1, M\})\delta^2
\end{align}
\end{proposition}
\begin{proof}
The proof can be found in Appendix \ref{appexprop1} of supplementary materials.
\end{proof}
Proposition \ref{prop1} shows that FIFO update rule can reduce the distribution discrepancy in the local cache by a factor at most $\min\{\frac{2\Gamma +1 }{M}, 1\}$ compared to the short-term label distribution depending on how the short-term label distributions are temporally correlated (i.e., $\Gamma$) and the size of the cache (i.e., $M$). In particular, if the short-term label distributions are independent across time (i.e., $\Gamma = 0$), then FIFO is able to reduce distribution discrepancy by $1/M$. Moreover, as the local cache size increases to infinity, the discrepancy diminishes asymptotically, i.e., $\lim_{M \to\infty} \mathbb{E}[(v^{k,r}_t - \pi^{k,r})^2] \to 0$. 

An obvious issue with the FIFO update rule is that the cached distribution $\v^k_t$ can still fluctuate significantly because of the short-term non-stationarity and the finite cache size, especially when the short-term label distribution is strongly temporally correlated and the cache size is small. Next, we propose two new cache update rules tailored to SFL. 

\subsubsection{Static Ratio Selective Replacement (SRSR)}
The reason why FIFO may result in a large fluctuation in the short-term label distribution is that it is unable to use historical data instances and their label distribution information. The goal of SRSR is to smooth out the short-term label distribution and make it approximate the long-term label distribution by using a moving average type of update rule. Specifically, SRSR comprises two steps. 

\textbf{Step 1}. SRSR computes a weighted average of the number of data instances with label $r$ in the local cache and that of the newly received data, i.e.,
\begin{align}
     \tilde{n}_r = (1 - \frac{B_s}{B}\theta) n_r(\mathcal{L}^k_{t-1}) + \theta n_r(\mathcal{S}^k_t) \label{srsr1}
\end{align}
This will be the target number of data instances with the label $r$ in the updated local cache. Here, the scalar $B_s/B$ ensures that the size constraint of the local cache is always satisfied since one can easily verify that for any $\theta \in [0, 1]$ we have
\begin{align}
    \sum_{r=1}^R \tilde{n}_r = (1-\frac{B_s}{B}\theta)B + \theta B_s = B
\end{align}

\textbf{Step 2}. SRSR performs selective replacement to meet the target label numbers while utilizing the new data as much as possible. Specifically, there are two cases depending on the values of $\tilde{n}_r$ and $n_r(\mathcal{S}^k_t)$. 
\begin{enumerate}
    \item Case 1: $\tilde{n}_r \leq  n_r(\mathcal{S}^k_t)$. In this case, SRSR removes all data with label $r$ in $\mathcal{L}^k_{t-1}$ and uniformly randomly selects $\tilde{n}_r$ data instances with label $r$ from $\mathcal{S}^k_t$ to insert into the local cache. 
    \item Case 2: $\tilde{n}_r > n_r(\mathcal{S}^k_t)$. In this case, SRSR uniformly randomly removes $n_r(\mathcal{L}^k_{t-1}) + n_r(\mathcal{S}^k_t) - \tilde{n}_r$ existing data instances with label $r$ from the local cache and inserts all data instances with label $r$ from $\mathcal{S}^k_t$ into the local cache. 
\end{enumerate}
Since the target label numbers are met, the cached label distribution by using SRSR is thus
\begin{align}
    v^{k,r}_t &= \frac{1}{B}\left((1-\frac{B_s}{B}\theta)n_r(\mathcal{L}^k_{t-1}) + \theta n_r(\mathcal{S}^k_t)\right) \notag\\
    &= \frac{\theta}{B}\sum_{\tau = 0}^t (1-\frac{B_s}{B}\theta)^\tau n_r(\mathcal{S}^k_{t-\tau})
\end{align}
The above equation shows that the cached label distribution takes all historical data distribution into account but discounts old information at a rate $1-\frac{B_s}{B}\theta$. 

\begin{proposition}\label{prop2}
    The discrepancy between the cached label distribution and the long-term label distribution by using SRSR is bounded as follows
    \begin{align}
         &\mathbb{E}[(v^{k,r}_t - \pi^{k,r})^2] \leq 2 (1-\frac{\theta}{M})^{2t} (\pi^{k,r})^2 \notag\\
         & + 2(1-(1-\frac{\theta}{M})^{2t})\frac{(1-\frac{\theta}{M})^{-\Gamma} - (1-\frac{\theta}{M})^{\Gamma+1}}{2 - \frac{\theta}{M}}\delta^2
    \end{align}
\end{proposition}
\begin{proof}
 The proof can be found in Appendix \ref{appexprop2} of supplementary materials.
\end{proof}
\begin{corollary}\label{coro11}
By choosing $\theta$ sufficiently small, the bound on $\mathbb{E}[(v^{k,r}_t - \pi^{k,r})^2]$ decreases over $t$. Moreover, 
\begin{align}
\lim_{t\to \infty}\mathbb{E}[(v^{k,r}_t - \pi^{k,r})^2] \leq 2\frac{(1-\frac{\theta}{M})^{-\Gamma} - (1-\frac{\theta}{M})^{\Gamma+1}}{2 - \frac{\theta}{M}}\delta^2
\end{align}
\end{corollary}
\begin{proof}
The proof can be found in Appendix \ref{appexcoro11} of supplementary materials.
\end{proof}
Proposition \ref{prop2} and Corollary \ref{coro11} imply that by choosing a small $\theta$ and with a large cache size $M$, SRSR can achieve a small label distribution discrepancy after sufficiently many rounds. On the other hand, the convergence to that small discrepancy is slower with a smaller $\theta$. Moreover, even in the limit $t \to \infty$, the discrepancy bound does not vanish unless $M \to \infty$, i.e., the local cache has an infinity capacity. 

\subsubsection{Dynamic Ratio Selective Replacement (DRSR)}
Now, we propose the DRSR update rule that overcomes the drawbacks of FIFO and SRSR. The goal of DRSR is to maintain the cached label distribution in the cache as the time-average short-term label distribution up to the current period. To this end, DRSR first uses a dynamic weight to compute the target numbers of data instances with different labels following a formula similar to Eq.\eqref{srsr1} in SRSR, i.e.,
\begin{align}
     \tilde{n}_r = (1 - \frac{B_s}{B}\theta_t) n_r(\mathcal{L}^k_{t-1}) + \theta_t n_r(\mathcal{S}^k_t) \label{drsr1}
\end{align}
where $\theta_1, ..., \theta_t$ is the sequence of dynamic weights. Once $\tilde{n}_r, \forall r$ is computed, DRSR follows the exact same Step 2 as in SRSR to perform the selective replacement. 

\begin{proposition}\label{prop3}
    The discrepancy between the cached label distribution and the long-term label distribution by using DRSR with $\theta_t = \frac{B}{B_s t}$ is bounded as follows
    \begin{align}
         \mathbb{E}[(v^{k,r}_t - \pi^{k,r})^2] \leq \frac{2\Gamma + 1}{t}\delta^2
    \end{align}
\end{proposition}
\begin{proof}
 The proof can be found in Appendix \ref{appexprop3} of supplementary materials.
\end{proof}
Proposition \ref{prop3} shows that the discrepancy decreases over time, and the cached label distribution converges to the long-term label distribution at a rate of $O(1/t)$. 

\section{Convergence Analysis}
In this section, we analyze the convergence of SFL. Because of the mismatch between the long-term label distribution $\boldsymbol{\pi}^k$ and the cached label distribution $\v^k_t$ of the local cache $\mathcal{L}^k_t$, the gradient $g^k_{t, \tau} = \nabla F^k(w^k_{t,\tau}; \mathcal{L}^k_t)$ computed in the local model update steps differs from the desired gradient on the long-term label distribution, i.e. $\nabla f^k(w^k_{t, \tau})$. Thanks to the full batch gradient descent, we are able to characterize the difference between $g^k_{t, \tau}$ and $\nabla f^k(w^k_{t,\tau})$ through an intermediate variable, which we name the \textit{virtual local gradient} and denote as $\hat{g}^k_{t, \tau}$. Specifically, $\hat{g}^k_{t, \tau}$ is defined as follows:
\begin{align}
    \hat{g}^k_{t, \tau} = \sum_{r=1}^R \pi^{k, r} \nabla F^k_r(w^k_{t,\tau}; \mathcal{L}^{k,r}_t)
\end{align}
where $\nabla F^k_r(w^k_{t,\tau}; \mathcal{L}^{k,r}_t)$ is the gradient computed on only the subset of data instances with label $r$, denoted by $\mathcal{L}^{k,r}_t$, in the current local cache $\mathcal{L}^k_t$. We note that $\hat{g}^k_{t, \tau}$ is only imaginary since neither it is actually computed nor it can be realistically computed. This is because our algorithm does \textit{not} actually divide $\mathcal{L}^k_t$ into $R$ subsets $\mathcal{L}^{k,1}, ..., \mathcal{L}^{k,R}$ and compute the gradients on each of these sets. Instead, only a single local gradient $\nabla F^k(w^k_{t, \tau}; \mathcal{L}^k_t)$ is computed. More critically, even with $\nabla F^k_r(w^k_{t,\tau}; \mathcal{L}^{k,r}_t), \forall r = 1, ..., R$, computing $\hat{g}^k_{t, \tau}$ requires the knowledge of the long-term label distribution $\boldsymbol{\pi}^k$, which is unknown by the algorithm. 

Before we move on to establish the connection between $g^k_{t, \tau}$ and $\nabla f^k(w^k_{t, \tau})$ and prove the convergence of the proposed SFL algorithm under different local cache update rules, we make the following standard assumptions. 

\begin{assumption}[Lipschitz Smoothness]
The local objective function is Lipschitz smooth, i.e., $\exists L > 0$, such that $\|\nabla f^k(x) - \nabla f^k(y)\| \leq L\|x - y\|$, $\forall x, y \in \mathbb{R}^d$ and $\forall k$.
\label{assm:smoothness}
\end{assumption}

\begin{assumption}[Unbiased Gradient Estimator]
For each client $k$, the label-wise local gradient is unbiased, i.e., $\mathbb{E}_{\mathcal{L}^{k,r}} \nabla F^k(x; \mathcal{L}^{k,r}) = \nabla f^{k,r}(x) \triangleq \mathbb{E}_{\xi^{r}} \nabla F(x; \xi^{r})$ where $\zeta^{r}$ is a instance with label $r$.
\label{assm:unbiased-local}
\end{assumption}

\begin{assumption}[Bounded Dissimilarity]
There exists constants $\sigma_G > 0$ and $A \geq 0$ so that
\begin{align}
    \|f^k(x)\|^2 \leq (A^2 + 1)\|f(x)\|^2 + \sigma^2_G, \forall x, \forall k
\end{align}
when the local loss functions are identical, $A^2=0$ and $\sigma^2_G = 0$.
\label{assm:variability}
\end{assumption}

\begin{assumption}[Gradient Bound]
The label-wise local gradient is bounded, 
\begin{align}
    \mathbb{E}\left[\|\nabla F^k(x; \mathcal{L}^{k,r})\|^2\right] \leq \sigma^2_M, \forall k, \forall r, \forall \mathcal{L}^{k,r}
\end{align}
\label{assm:second-moment}
\end{assumption}
Similar assumptions are commonly used in both the non-convex optimization and FL literature \cite{yang2021achieving, yang2022anarchic,jhunjhunwala2022fedvarp, wang2022friends}. We adapted some of the assumptions for the label-wise local gradient. 

In the previous section, we established the upper bound on the cached label distribution and the long-term label distribution for different local update rules. To facilitate the exposition, we introduce a unified notation $\lambda_t$ to represent the upper bounds. Specifically, 
\begin{align}
    \mathbb{E}[(v^{k,r}_t - \pi^{k,r})^2] \leq \lambda^2_t
\end{align}
The specific forms of $\lambda_t$ can be found in Propositions \ref{prop1}, \ref{prop2} and \ref{prop3} for FIFO, SRSR and DRSR, respectively. 

We begin by introducing some necessary lemmas to help us with the theorem that follows.
\begin{lemma}
The expectations of the difference between the \textit{real local gradient} $g_{t, \tau}^k$ and \textit{virtual local gradient} $\hat{g}_{t, \tau}^k$ is upper bounded as:
\begin{align}
 \mathbb{E} [| g_{t, \tau}^k - \hat{g}_{t, \tau}^k |^2] \leq R^2 \lambda^2_t \sigma_M^2
\end{align}
The difference between the \textit{virtual local gradient}  $\hat{g}_{t, \tau}^k$ and expected gradient $\nabla f^{k}(w^k_{t,\tau})$ is upper bounded as:
\begin{align}
 \mathbb{E} [| \hat{g}_{t, \tau}^k - \nabla f^{k}(w^k_{t,\tau})|^2] \leq  2 R^2 \overline{\pi}^2 \sigma_M^2
\end{align}
where $\overline{\pi} = \max_{k,r} \pi^{k, r}$. 
\label{lemma1}
\end{lemma}
\begin{proof}
The proof can be found in Appendix \ref{appexlemma1} of supplementary materials.
\end{proof}

The following result is on the upper bound for the $\tau$-step SGD in the full participation case with Lemma \ref{lemma1}.
\begin{lemma}\label{lemma2}
For any step-size satisfying $\eta_L \leq \frac{1}{8LE}$, we have: $\forall \tau = 0, ..., E-1$
\begin{align}
    &\mathbb{E}[\|w^k_{t, \tau} - w_t\|^2]  \leq 5 E \eta_L^2 R^2 \lambda^2_t \sigma_M^2  +  60E^2\eta_L^2 R^2 \overline{\pi}^2 \sigma_M^2 \notag\\
    &+ 30E^2 \eta_L^2 \sigma^2_G + 30 E^2 \eta_L^2 (A^2 +1) \|\nabla f(x) \|^2
\end{align}
\end{lemma}
\begin{proof}
The proof can be found in Appendix \ref{appexlemma2} of supplementary materials.
\end{proof}
By defining $\Delta_t = \bar{\Delta}_t + e_t$, where $\Delta_t = -\frac{1}{K}\sum_{k=1}^K\sum_{\tau = 0}^{E-1}g_{t,\tau}^k$ and  $\bar{\Delta}_t = -\frac{1}{K}\sum_{k=1}^K\sum_{\tau = 0}^{E-1}\hat{g}_{t, \tau}^k$ we can obtain the convergence bound of SFL with full client participation as follows: 
\begin{theorem}\label{thm1}
Let constant local and global learning rates $\eta_L$ and $\eta$ be chosen as such that $ \eta_L \leq \min\left (  \frac{1}{\sqrt{60 (A^2+1) }EL},  \frac{1}{8LE} \right )$ and $\eta\eta_L \leq \frac{1}{4EL}$. Under Assumption \eqref{assm:smoothness}-\eqref{assm:second-moment} with full client participation, the sequence of model $w_t$ in real sequence satisfies
\begin{align}
    \min_{t = 0,..., T-1}\mathbb{E}\|\nabla f(w_t)\|^2 \leq \frac{f_0 - f_*}{c\eta\eta_L E T} + \Phi_G + \Phi_M + \Phi_L
\end{align}
where $c$ is a constant, $f_0 \triangleq f(w_0)$, $f_* \triangleq f(w_*)$, $w_*$ is the optimal model and
\begin{align}
    &\Phi_G = \frac{30 E^2 \eta_L^2 L^2}{c}\sigma_G^2\\
    &\Phi_M = \frac{60 \eta^2_L E^2 L^2 R^2 \overline{\pi}^2 }{c}\sigma^2_M \\
    &\Phi_L = \frac{\left (5 \eta^2_L E L^2 + 3 \eta \eta_L L  E + 1\right ) R^2 \sigma_M^2}{c T} \sum_{t=0}^{T-1}\lambda^2_t
\end{align}
\end{theorem}
\begin{proof}
The proof can be found in Appendix \ref{appdex:theorem1} of supplementary materials.
\end{proof}
The above convergence bound contains four parts: a vanishing term $\frac{f_0 - f_*}{c\eta\eta_L E T}$ as $T$ increases, a constant term $\Phi_G$ whose size depends on the problem instance parameters and is independent of $T$, a third term $\Phi_M$ is affected by the number of classes $R$ and maximum ratio $\overline{\pi}$, and a final term $\Phi_L$ that depends on the cumulative gap between the real and virtual sequences. The key insight derived by Theorem \ref{thm1} is that the SFL convergence bound depends on two additional terms $\Phi_M$ and $\Phi_L$ when compare to the conventional FL. For each client, if we could use the long-term label distribution, the cumulative ratio gap $\frac{1}{T}\sum_{t=0}^{T-1}\lambda^2_t=0$.  Consequently, the convergence bound is simply $\frac{f_0 - f_*}{c\eta\eta_L E T} + \Phi_G + \Phi_M$. However, this gap cannot be eliminated since the client cannot directly use the long-term label distribution in the local model updating. By applying the specific learning rate, with $T \to \infty$, we can get the following corollary for the general convergence rate:
\begin{corollary}\label{corollary1}
With learning rates $\eta_L = \frac{1}{\sqrt{T}E}$ and $\eta = \sqrt{EK} $, the convergence rate of the general case under full client participation is:
{\begin{align}\label{coro1}
& \mathcal{O}(\frac{1}{\sqrt{EKT}}) +  \underbrace{\mathcal{O}(\frac{\sigma_G^2}{T})}_{\Phi_G} +  \underbrace{\mathcal{O}(\frac{ R^2 \overline{\pi}^2 \sigma_M^2}{T})}_{\Phi_M} +  \underbrace{\mathcal{O}(\frac{\sigma_M^2 \sum_{t=0}^{T-1} \lambda^2_t}{T})}_{\Phi_L} \notag
\end{align}}
\end{corollary}
Based on the corollary \ref{corollary1} above, $\Phi_L$ is the major factor that determines whether the results converge to a stationary point without any constant terms. By substituting the values of $\lambda^2_t$ for the three update rules mentioned earlier, we can derive the corresponding final convergence rates. Furthermore, it is shown that under the \textbf{DRSR} update rule the SFL can eventually converge to a stationary point.

\section{Experiments}
\textbf{Setup}. Our experiments are based on two datasets: FMNIST and the network traffic classification dataset (NTC) extracted from ISCXVPN2016 as in \cite{zhang2020autonomous}. FMNIST is a commonly used dataset for image classification tasks, while NTC is a specialized dataset for network traffic classification. It contains 45000 network packets that are divided into 10 classes, each representing a different application such as YouTube or Skype, which are encrypted traffic samples using various methods. The packet vectors can be reshaped to 39 × 39 bytes gray images. For both datasets, we use LeNet \cite{lecun1998gradient} as the backbone model, specifically modified according to the different datasets. All experiment results reported are the average of 10 independent runs.

\textbf{Data Stream Generation}. The SFL system consists of 10 clients, and every client receives training data samples from $C$ classes in the long-term distribution, which are non-i.i.d between clients. To simulate the time-varying short-term distributions, we generate 10 possible distributions for each client. In each time slot, the client receives one distribution as the short-term distribution. To capture the temporal correlation of the short-term distribution, the transition between any two short-term distributions is governed by a probability determined by the Kullback–Leibler (K-L) divergence \cite{joyce2011kullback} between these two label distributions. The probability of distribution transition between two distributions is higher when the K-L divergence between them is lower. The long-term distribution is obtained as the stationary distribution of these short-term distributions, based on the transition matrix, which is unknown to the client in advance.


\textbf{Benchmarks}. In the experiment, the following two benchmarks are used for performance comparison. 
\begin{enumerate}
    \item \textbf{Full Information (FULL)}. In this ideal scenario, each client has a local training dataset with a distribution the same as the long-term distribution. 
    \item \textbf{Lazy Updates (LAZY)}. In this scenario, the client keeps the initial training dataset in the cache and does not update its dataset. The client then conducts local training by utilizing this static local dataset.
\end{enumerate}

As we have analyzed in Section IV, both \textbf{FIFO} and \textbf{SRSR} update rules can converge to a stationary point with infinite cache capacity. However, since infinite cache capacity is impractical in real-world scenarios, we will only conduct experiments under finite cache capacity. 

\begin{figure}[h]
\centering
\subfloat[NTC Non-iid ($C = 3$)]{\includegraphics[width=0.49\linewidth]{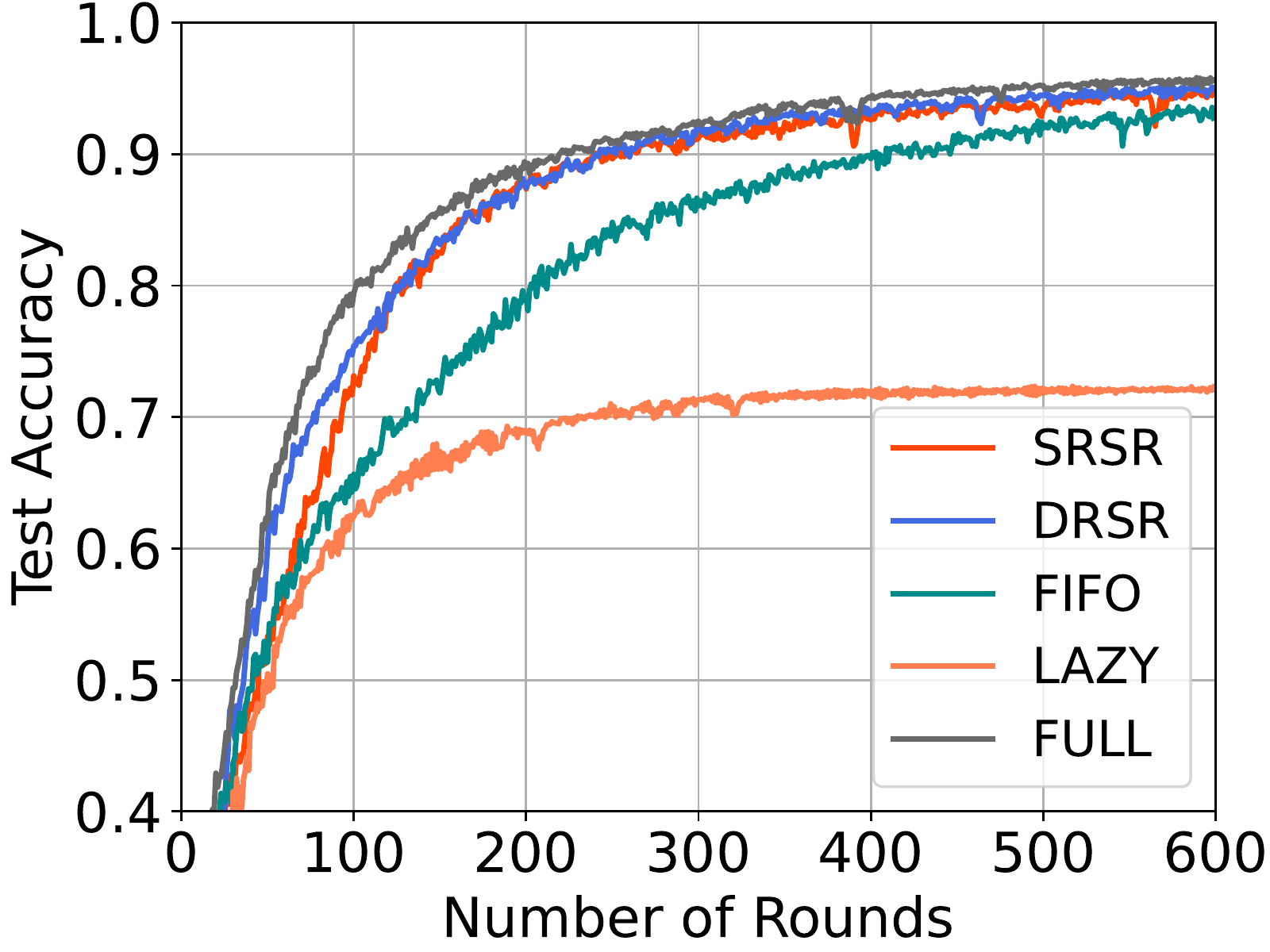}} 
\subfloat[FMNIST Non-iid ($C = 3$)]{\includegraphics[width=0.49\linewidth]{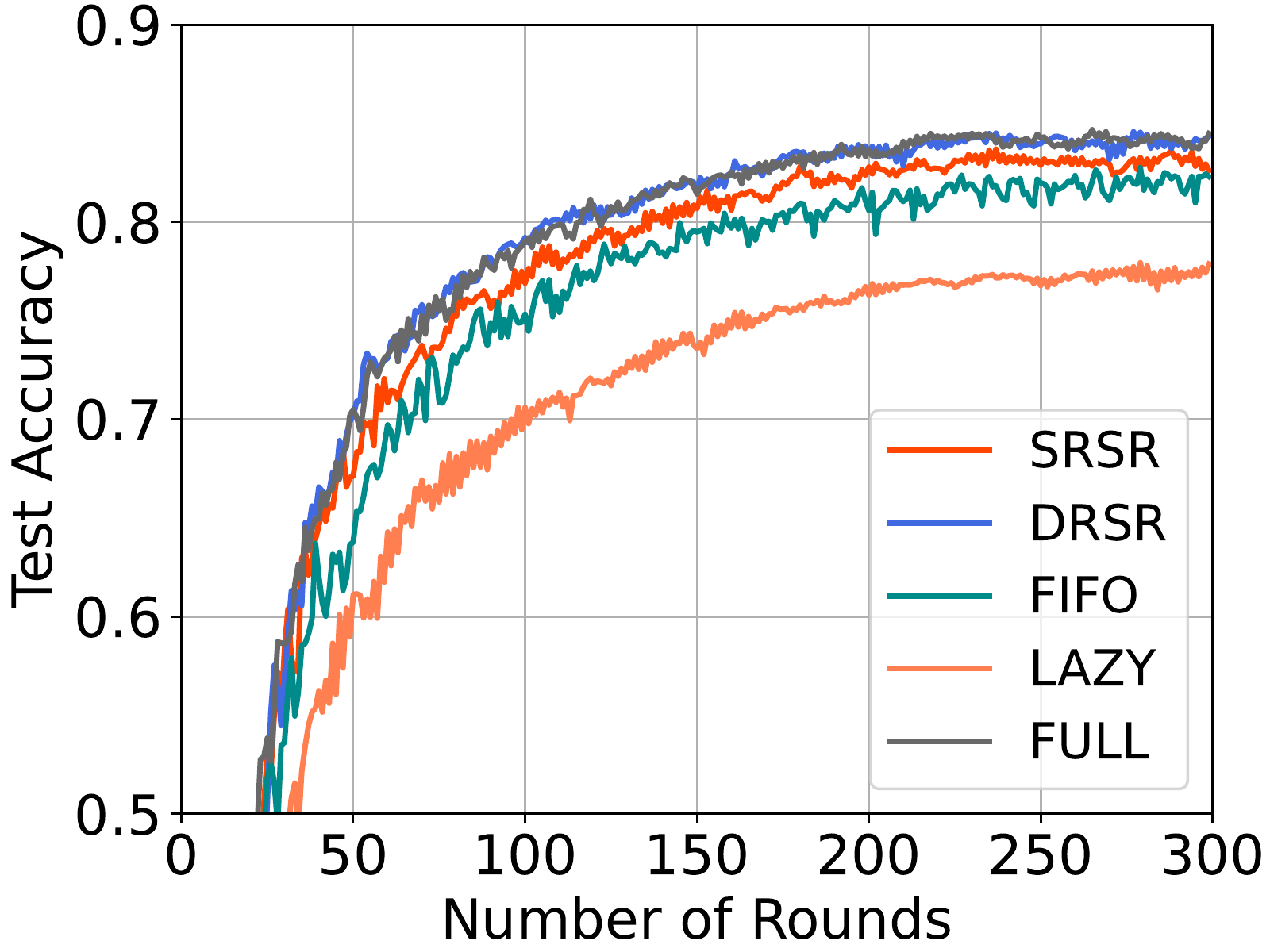}} 
\caption{Performance comparison of the proposed update rules and benchmarks with full participation.} \label{sfl-cpu}
\end{figure}

\textbf{Performance comparison}. We first compare the convergence performance between our proposed update rules and benchmarks with parameters $\left \{ B = 300, B_s =150, C =3, \theta= \frac{2}{3} \right \}$. Fig.~\ref{sfl-cpu}(a) and (b) plot the convergence curves on the NTC dataset and the FMNIST dataset with full client participation, respectively. Several observations are made as follows. First,  \textbf{DRSR}, \textbf{SRSR} and \textbf{FIFO} outperform \textbf{LAZY} in terms of test accuracy and convergence speed, particularly in the later stages. \textbf{DRSR} and \textbf{SRSR} achieve performance close to \textbf{FULL} on the NTC dataset due to their ability to gradually approximate the long-term label distribution. Second, \textbf{DRSR} and \textbf{SRSR} outperform \textbf{FIFO} in the entire learning process, mainly attributed to their ability to retain the knowledge of past data streams. Third, the learning performance of \textbf{DRSR}, \textbf{SRSR} and \textbf{FIFO} is significantly better than \textbf{LAZY} on the NTC dataset, while the performance gain is less significant on the FMNIST dataset. Overall, \textbf{DRSR} and \textbf{SRSR} are better than \textbf{FIFO} on both datasets. More comparisons between \textbf{DRSR} and \textbf{SRSR} will be given later.

\begin{figure}[h]
\centering
\subfloat[The per-slot discrepancy $\psi_t$]{\includegraphics[width=0.49\linewidth]{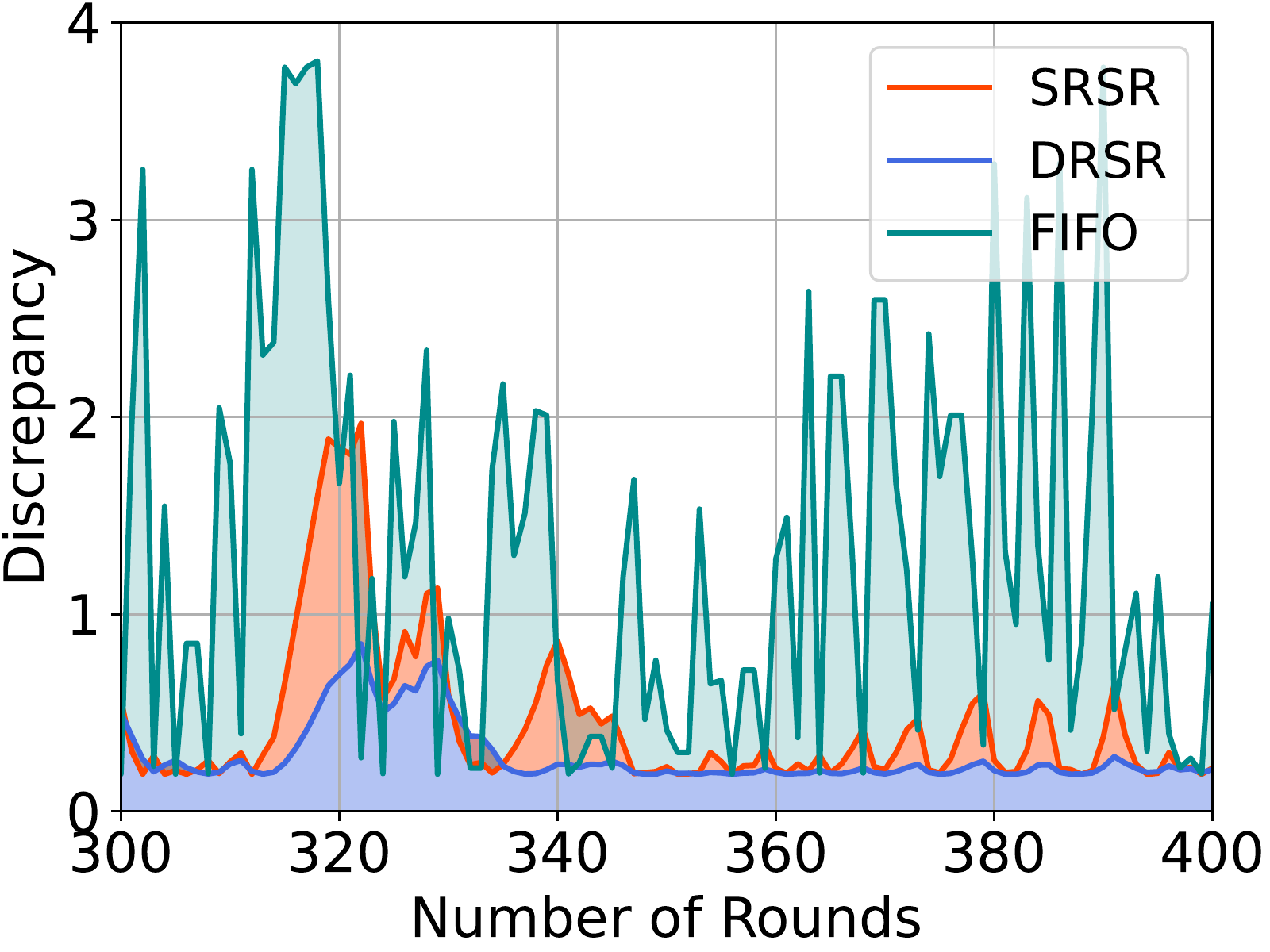}} 
\subfloat[The accumulated discrepancy $\psi$]{\includegraphics[width=0.49\linewidth]{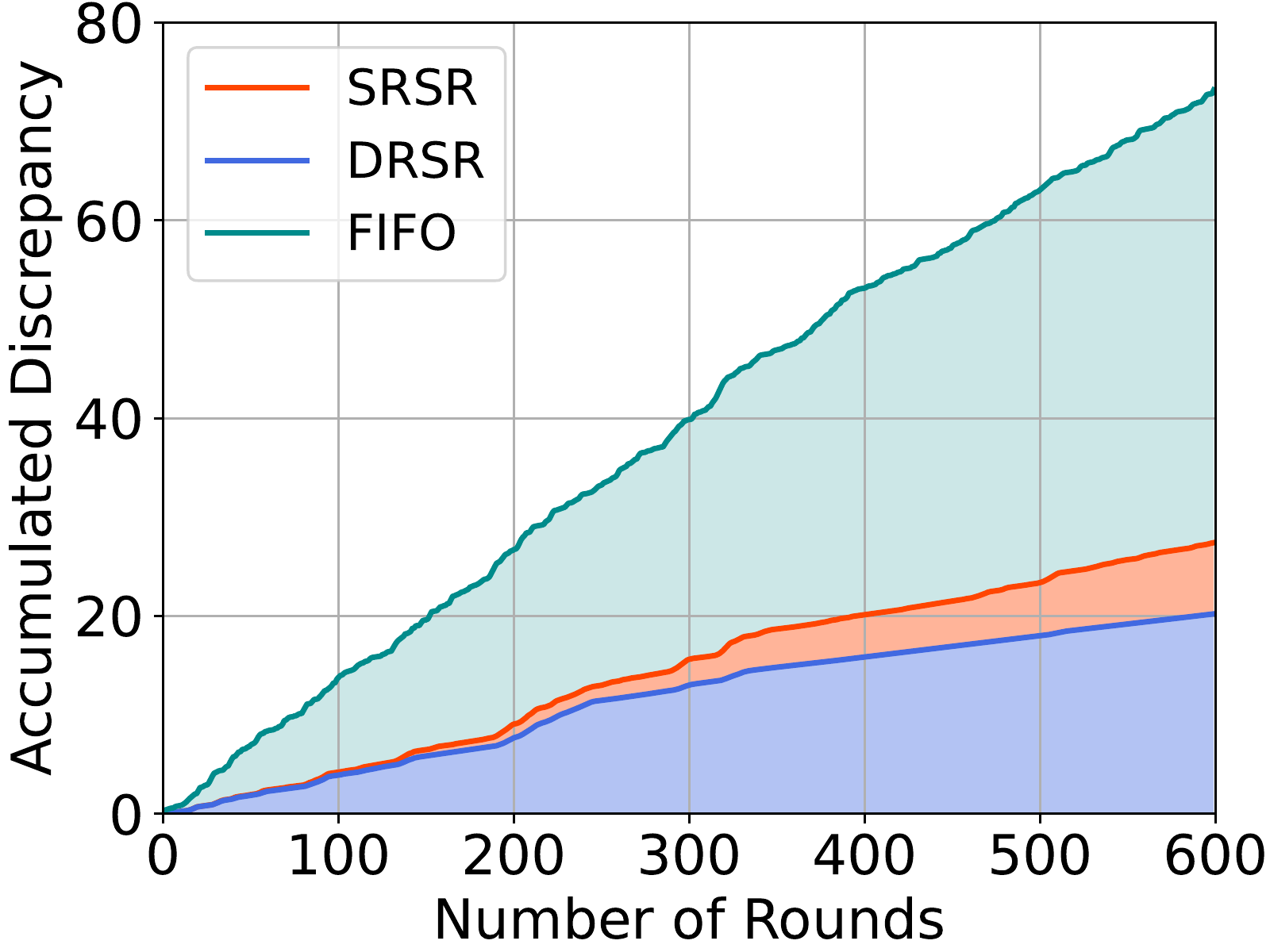}} 
\caption{Discrepancy between the cached label distribution ${v}^{k, r}_{t}$ and the long-term label distribution ratio $\pi^{k, r}$.}  \label{sfl-tg} 
\end{figure}

\textbf{Distribution discrepancy}. The learning performance of \textbf{SFL} depends on how well it can approximate the long-term label distribution. In this part, we examine how the distribution discrepancy changes during the training process for different local dataset update rules. The per-slot discrepancy is defined as $\psi_t = \sum_{k \in K} \sum_{r \in R} ( {v}^{k, r}_{t} -  \pi^{k, r} )^2$ and the accumulated discrepancy is defined as $\psi = \sum_{t \in T} \sum_{k \in K} \sum_{r \in R} ( {v}^{k, r}_{t} -  \pi^{k, r} )^2$. Fig.~\ref{sfl-tg} shows the per-slot discrepancy and the accumulated discrepancy for the different update rules. From Fig.~\ref{sfl-tg}(a), we can see that \textbf{DRSR} and \textbf{SRSR} exhibit lower discrepancy and fewer fluctuations compared to \textbf{FIFO}, and the discrepancy of \textbf{DRSR} decreases over time. Fig.~\ref{sfl-tg}(b) demonstrates that \textbf{DRSR} has the lowest cumulative discrepancy throughout the learning process and performs better than both \textbf{SRSR} and \textbf{FIFO}. 

\begin{figure}[h]
\centering
\subfloat[NTC Non-iid ($B_s = 30$)]{\includegraphics[width=0.49\linewidth]{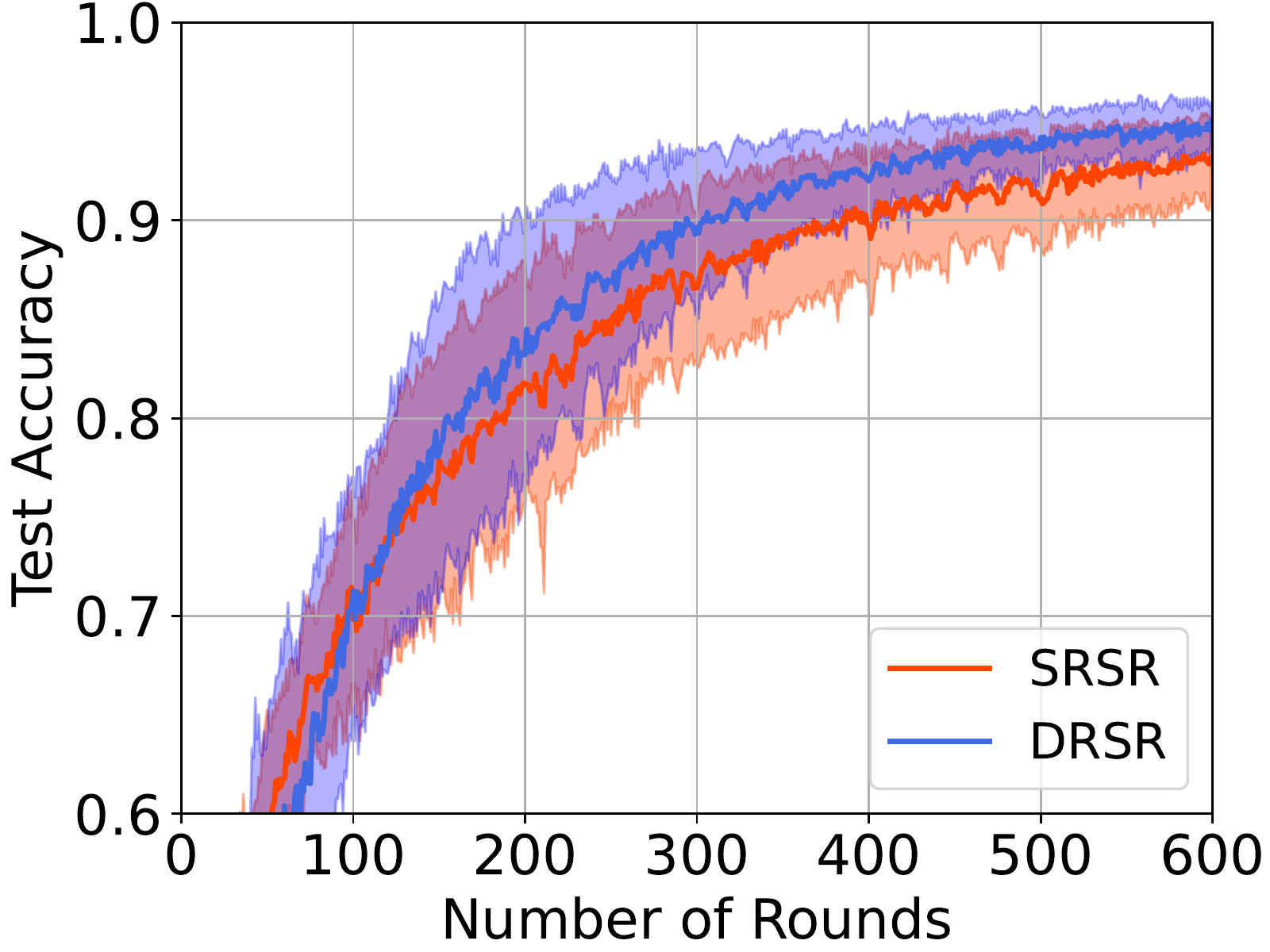}} 
\subfloat[NTC Non-iid ($B_s = 150$)]{\includegraphics[width=0.49\linewidth]{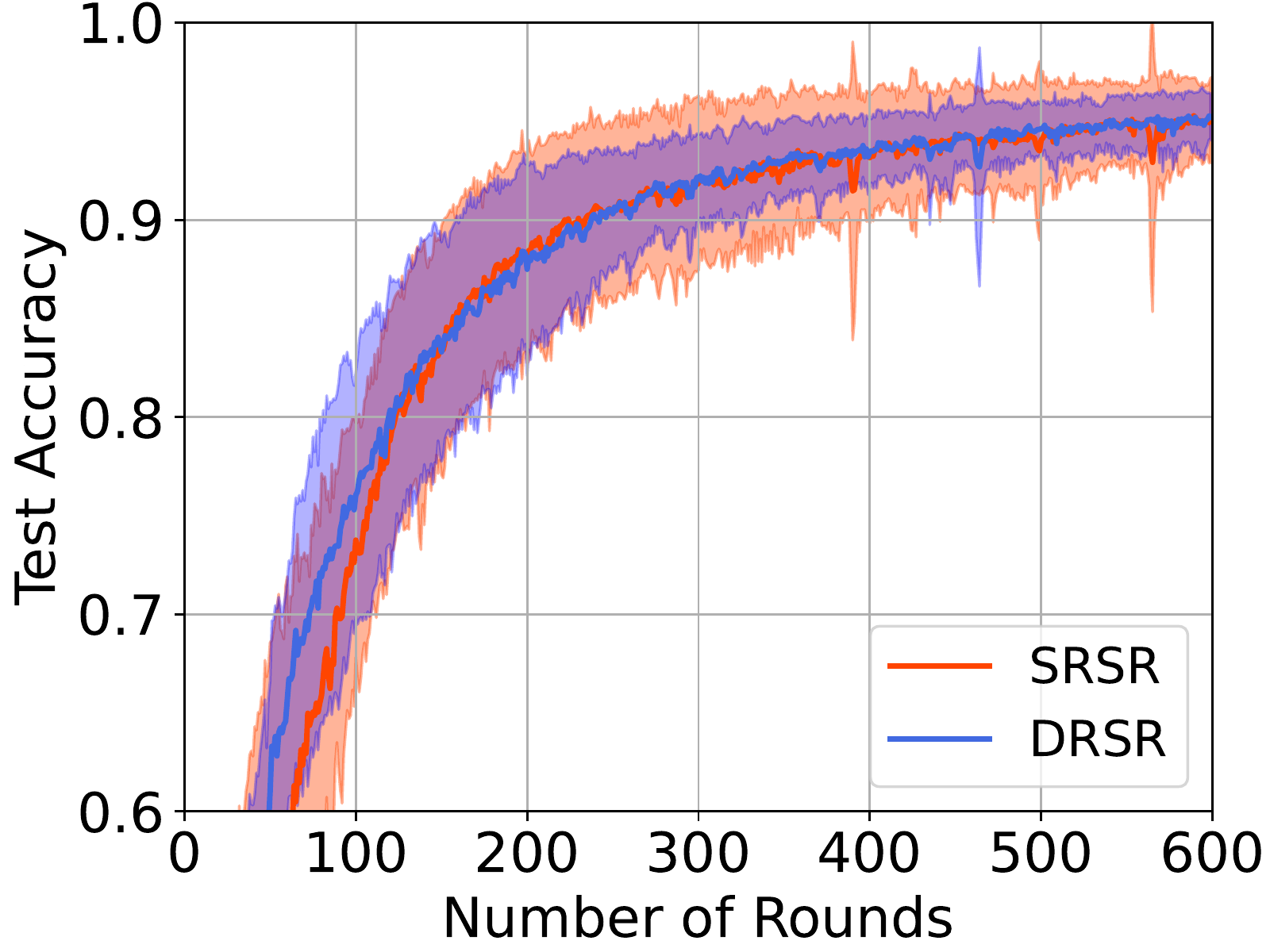}} 
\caption{Performance comparison of \textbf{SRSR} and \textbf{DRSR} with different $B_s$.}   \label{sfl-ssd}
\end{figure}

\textbf{Impact of streaming data size $B_s$}. The proposed update rules, \textbf{DRSR} and \textbf{SRSR} were examined for their learning performance when trained with different values of $B_s$ and a constant $B = 300$. A larger value of $B_s$ corresponds to a larger streaming packet per round. The NTC dataset was used with two different streaming data sizes, $B_s \in \left \{ 30, 150 \right \}$, to investigate the effects of varying $B_s$. The results, shown in Fig.~\ref{sfl-ssd}, indicate that both \textbf{DRSR} and \textbf{SRSR} are capable of achieving the desired level of test accuracy. The variance of \textbf{DRSR} decreases with training, particularly in the later stages, where it is significantly smaller than the variance of \textbf{SRSR}. 

\begin{figure}[h]
\centering
\subfloat[NTC Non-iid ($B = 100$)]{\includegraphics[width=0.49\linewidth]{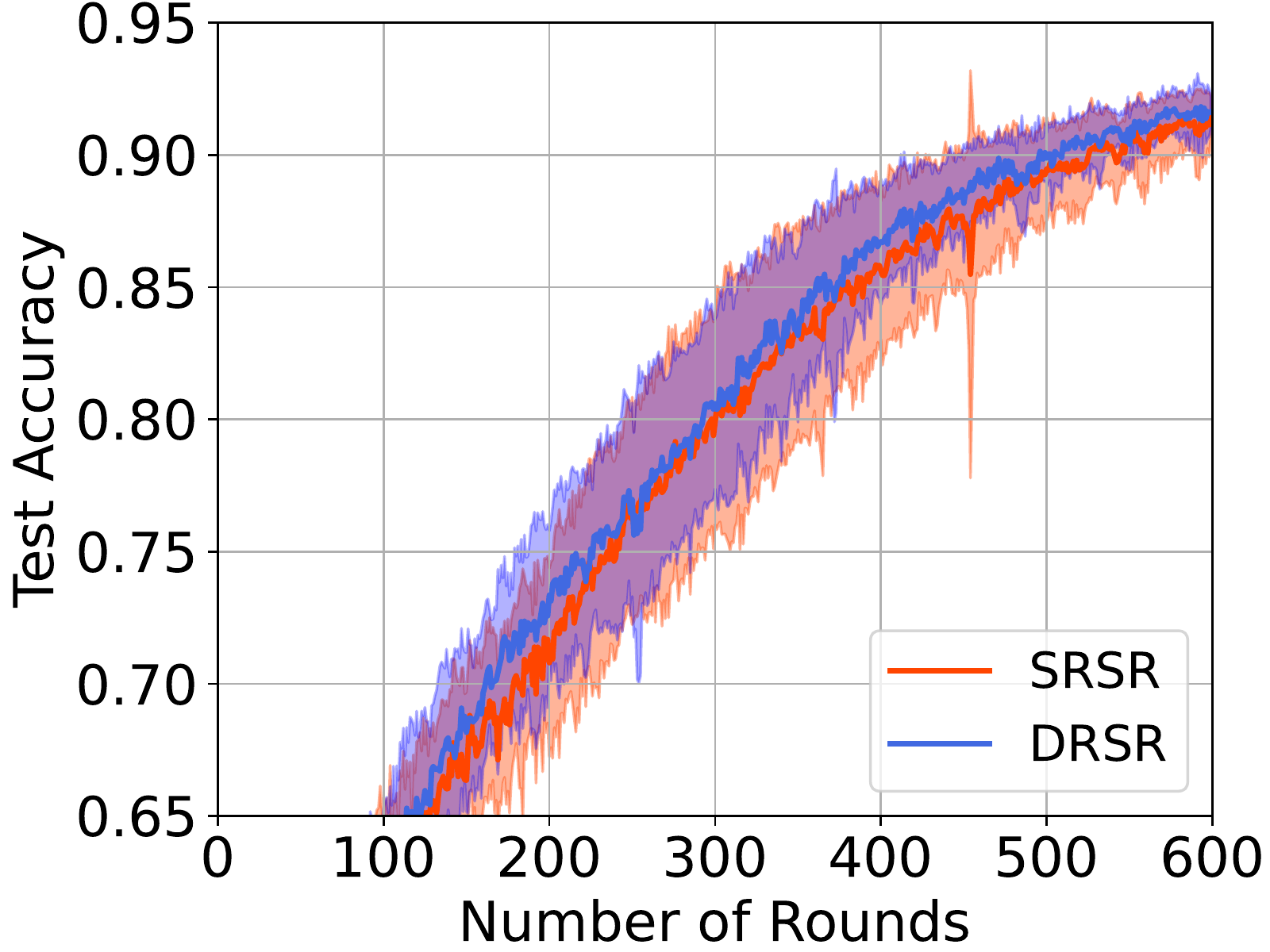}} 
\subfloat[NTC Non-iid ($B = 300$)]{\includegraphics[width=0.49\linewidth]{figure/Ntc_bs1.pdf}} 
\caption{Performance comparison of \textbf{SRSR} and \textbf{DRSR} with different $B$.}   \label{sfl-smax}
\end{figure}

\textbf{Impact of cache capacity $B$}. In this set of experiments, we investigate the impact of varying the cache capacity, represented by $B$, on the learning performance while keeping the ratio ($\frac{B_s}{B}$) constant at 0.1. We use the NTC dataset and tested two different values of $B \in \left \{ 100, 300 \right \}$. Our results, as shown in Fig.~\ref{sfl-smax}, suggest that increasing $B$ leads to higher test accuracy and faster training rates. For example, at 200 rounds, the test accuracy is 0.72 for $B$ = 100 and 0.83 for $B$ = 300. However, in practical situations, the cache capacity of a client is often limited, despite the potential for better performance with larger values of $B$.

\begin{figure}[h]
\centering
\subfloat[NTC Non-iid ($B_s = 30$)]{\includegraphics[width=0.49\linewidth]{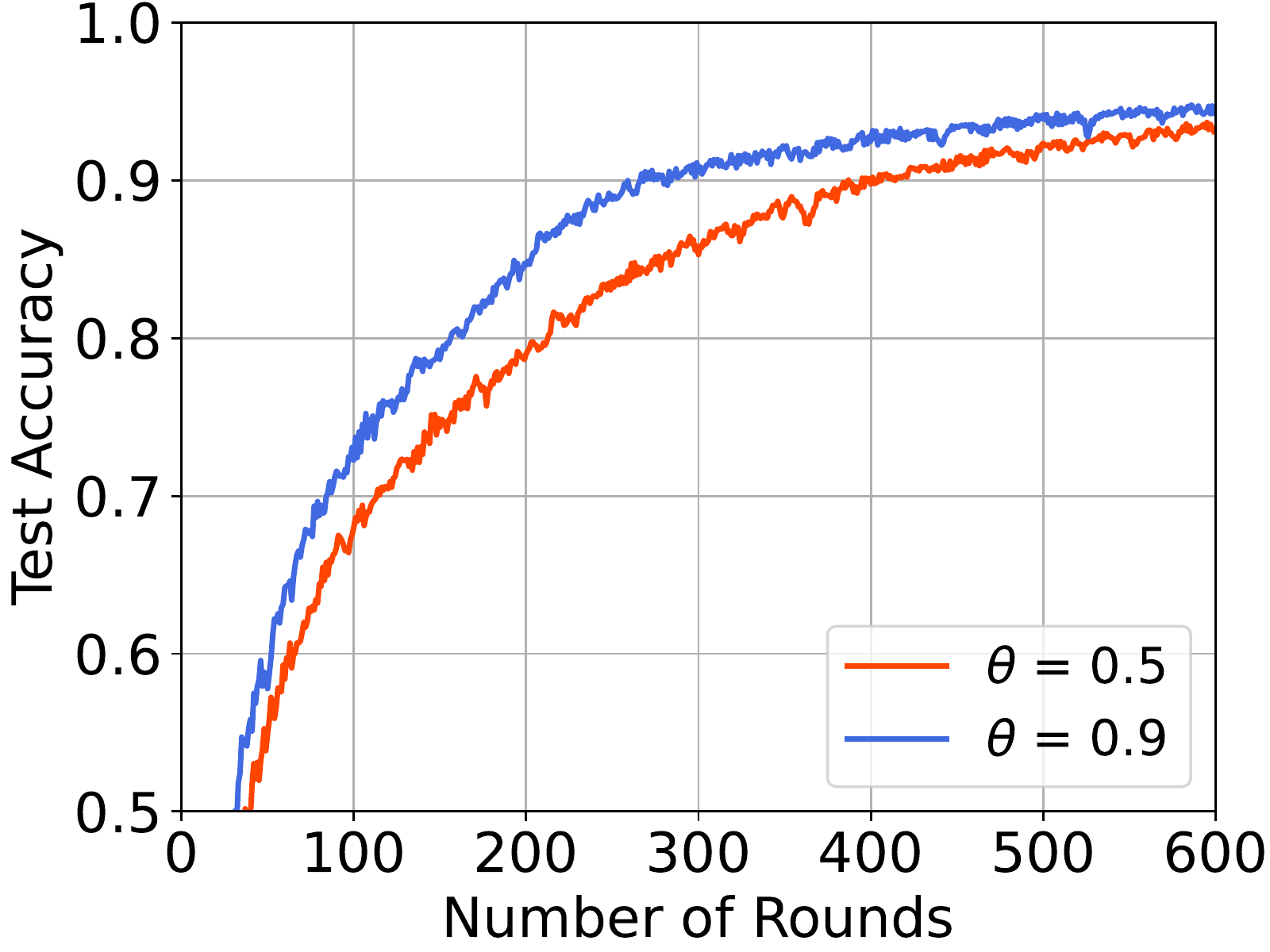}} 
\subfloat[NTC Non-iid ($B_s = 150$)]{\includegraphics[width=0.49\linewidth]{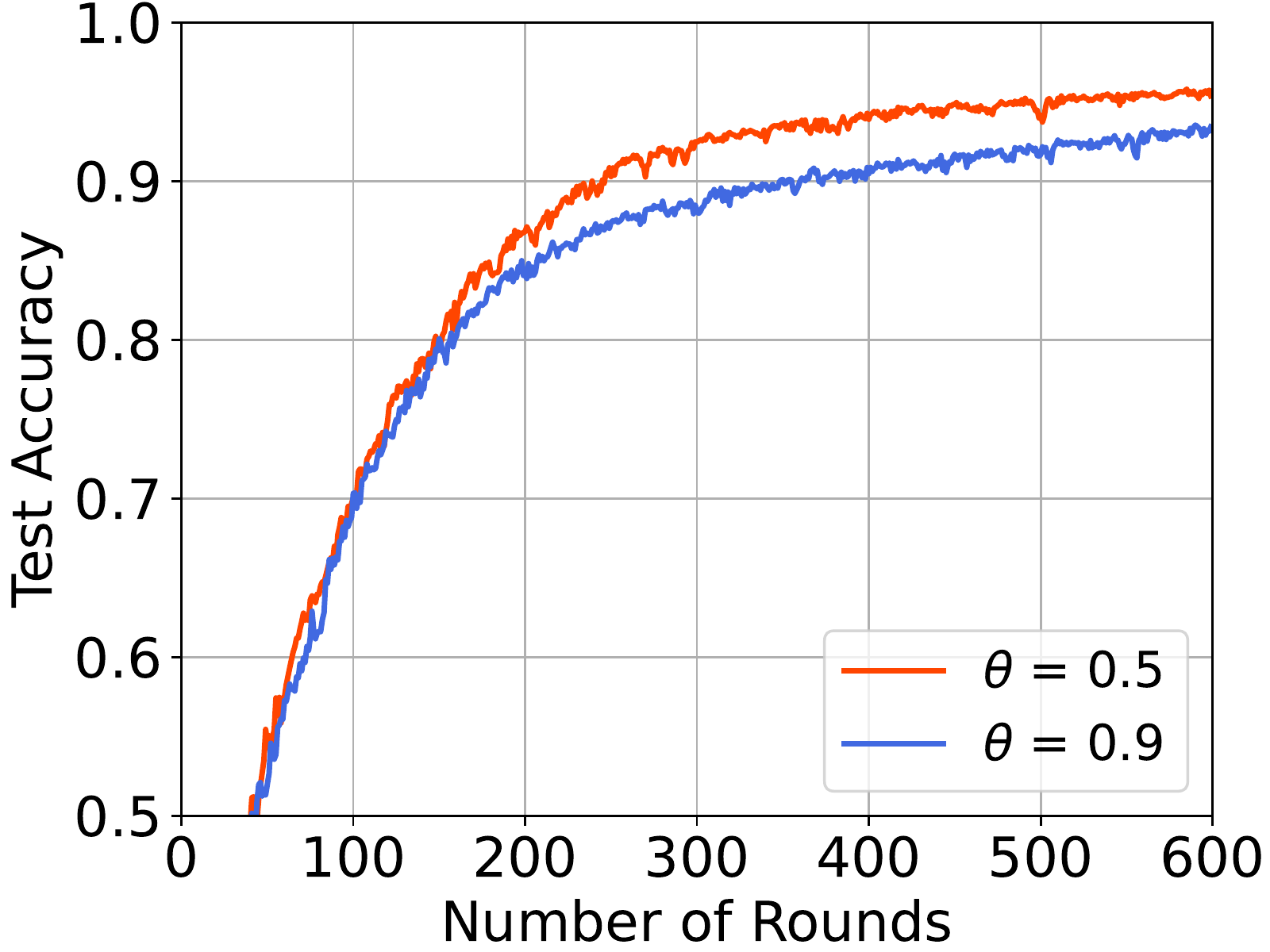}} 
\caption{Performance comparison of \textbf{SRSR} with different $\theta$.}   \label{sfl-theta}
\end{figure}

\textbf{Impact of parameter $\theta$}. In this section, we examine the impact of $\theta$, which tunes the amount of incoming data to put in the cache, on the learning performance of \textbf{SRSR}. The rest parameters are $\left \{ B = 300, B_s = [30, 150], C =3\right \}$. The experiment results are shown in Fig.~\ref{sfl-theta}. When $B_s$ is large, a smaller $\theta$ leads to better learning performance. This occurs because larger $\theta$ can cause a substantial shift in the cached label distribution, leading to an increase in variation (as seen in Fig.~\ref{sfl-theta}(b)). However, when $B_s$ is small, choosing a small value for $\theta$ leads to a degradation of learning performance. The reason for this is that the ratio in the cached label distribution changes at a slow pace at the beginning (as seen in Fig.~\ref{sfl-theta}(a)). Consequently, determining the appropriate $\theta$ beforehand is a challenging task. However, this issue can be resolved using the \textbf{DRSR} update rule, which reduces $\theta$ gradually over time.

\section{Conclusion}\label{section 7}
Our paper presents a novel Federated Learning (FL) framework named SFL, which differs from traditional FL by operating on a dynamic dataset. This dynamic nature of the data, coupled with the limited cache capacity on clients, results in discrepancies between the local training dataset and the long-term data distribution. We propose three update rules for the local cache update process in the SFL problem and provide a thorough theoretical analysis and experimental comparison to support our work. Future research will focus on developing more effective update rules for SFL to accelerate the training convergence speed and explore the potential of applying SFL to other practical scenarios.


\ifCLASSOPTIONcaptionsoff
  \newpage
\fi

\bibliographystyle{IEEEtran}
\bibliography{bibligraphy}

\newpage
\setcounter{page}{1}
\appendices
\section{Proof of Proposition \ref{prop1}}
\label{appexprop1}
First notice $v^{k,r}_t$ is essentially the average short-term label distribution of periods $t - M + 1$ through $t$, thus
\begin{align}
    &\mathbb{E}[(v^{k,r}_t - \pi^{k,r})^2] = \mathbb{E}[(\frac{1}{M}\sum_{m=1}^M u^{k,r}_{t-m+1} - \pi^{k,r})^2]\\
    &=\frac{1}{M^2}\sum_{m=1}^M \mathbb{E}[(u^{k,r}_{t-m+1} - \pi^{k,r})\sum_{m'=1}^M(u^{k,r}_{t-m'+1} - \pi^{k,r})]\\
    &\leq \frac{1}{M} (\min\{2\Gamma+1, M\})\delta^2
\end{align}

\section{Proof of Proposition \ref{prop2}}
\label{appexprop2}
The discrepancy can be bounded as follows. 
\begin{align}
    &\mathbb{E}[(v^{k,r}_t - \pi^{k,r})^2] =\mathbb{E}[(\frac{\theta }{M}\sum_{\tau=0}^t(1-\frac{\theta}{M})^\tau u^{k, r}_{t-\tau} - \pi^{k, r})^2]\\
    =&\mathbb{E}[(\frac{\theta}{M}\sum_{\tau=0}^t(1-\frac{\theta}{M})^\tau(u^{k, r}_{t-\tau} - \pi^{k, r}) + (1-\frac{\theta}{M})^t (-\pi^{k,r}))^2] \label{prop2.1}\\
    \leq & 2 \mathbb{E}[(\frac{\theta}{M}\sum_{\tau=0}^t(1-\frac{\theta}{M})^\tau(u^{k, r}_{t-\tau} - \pi^{k, r}))^2] + 2 (1-\frac{\theta}{M})^{2t} (\pi^{k,r})^2 \label{prop2.2}\\
    \leq & \frac{2\theta^2}{M^2}\sum_{\tau=0}^t (1-\frac{\theta}{M})^\tau \mathbb{E}[(u^{k,r}_{t-\tau}-\pi^{k,r}) \notag \\
     & \sum_{\tau'=0}^t (1-\frac{\theta}{M})^{\tau'}(u^{k,r}_{t-\tau'} - \pi^{k, r})] + 2 (1-\frac{\theta}{M})^{2t} (\pi^{k,r})^2\\
    \leq & \frac{2\theta^2}{M^2}\sum_{\tau=0}^t (1-\frac{\theta}{M})^\tau (1-\frac{\theta}{M})^\tau \sum_{i=-\Gamma}^\Gamma (1-\frac{\theta}{M})^i\delta^2 \notag \\
     & + 2 (1-\frac{\theta}{M})^{2t} (\pi^{k,r})^2 \label{prop2.3}\\
    =& 2 \frac{\theta}{M}\frac{1- (1-\frac{\theta}{M})^{2t}}{1-(1-\frac{\theta}{M})^2}\left((1-\frac{\theta}{M})^{-\Gamma} - (1-\frac{\theta}{M})^{\Gamma+1}\right)\delta^2 \notag \\
    & + 2 (1-\frac{\theta}{M})^{2t} (\pi^{k,r})^2\\
    =&2(1-(1-\frac{\theta}{M})^{2t})\frac{(1-\frac{\theta}{M})^{-\Gamma} - (1-\frac{\theta}{M})^{\Gamma+1}}{2 - \frac{\theta}{M}}\delta^2 \notag \\
    & + 2 (1-\frac{\theta}{M})^{2t} (\pi^{k,r})^2
\end{align}
where the Eq.\eqref{prop2.1} uses $\frac{\theta}{M}\sum_{\tau=0}^t(1-\frac{\theta}{M})^\tau + (1-\frac{\theta}{M})^t = 1$; Eq.\eqref{prop2.2} is a result of the triangle inequality; Eq.\eqref{prop2.3} uses Assumption \ref{assm1}.

\section{Proof of Corollary \ref{coro11}}
\label{appexcoro11}
It is easy to see that $\mathbb{E}[(v^{k,r}_t - \pi^{k,r})^2]$ is a weighted sum of $(\pi^{k,r})^2$ and $\frac{(1-\frac{\theta}{M})^{-\Gamma} - (1-\frac{\theta}{M})^{\Gamma+1}}{2 - \frac{\theta}{M}}\delta^2$ where the weight $(1-\frac{\theta}{M})^{2t}$ decreases with $t$. Moreover, it is easy to prove that $\frac{(1-\frac{\theta}{M})^{-\Gamma} - (1-\frac{\theta}{M})^{\Gamma+1}}{2 - \frac{\theta}{M}}\delta^2$ is increasing in $\theta/M$. Thus, by choosing $\theta$ sufficiently small, $\frac{(1-\frac{\theta}{M})^{-\Gamma} - (1-\frac{\theta}{M})^{\Gamma+1}}{2 - \frac{\theta}{M}}\delta^2$ can be made smaller than $(\pi^{k,r})^2$. Therefore, the weighted sum decreases with time and approaches $\frac{(1-\frac{\theta}{M})^{-\Gamma} - (1-\frac{\theta}{M})^{\Gamma+1}}{2 - \frac{\theta}{M}}\delta^2$ in the limit. 

\section{Proof of Proposition \ref{prop3}}
\label{appexprop3}
We bound the discrepancy as follows. First plugging $\theta_t = \frac{B}{B_s t}$ into Eq.\eqref{drsr1}, we have
\begin{align}
    \tilde{n}_r &= \frac{t-1}{t}n_r(\mathcal{L}^k_{t-1}) + \frac{B}{B_s t}n_r(\mathcal{S}^k_t) \\
    &= \frac{t-1}{t}\left(\frac{t-2}{t-1}n_r(\mathcal{L}^k_{t-2}) + \frac{B}{B_s (t-1)}n_r(\mathcal{S}^k_{t-1})\right) \notag \\
    &+ \frac{B}{B_s t}n_r(\mathcal{S}^k_t)\\
    &= \frac{B}{B_s t}\sum_{\tau=1}^t n_r(\mathcal{S}^k_{\tau})
\end{align}
so we can obtain $v^{k,r}_t = \frac{1}{t}\sum_{\tau = 1}^t u^{k,r}_t$, then we will get
\begin{align}
    &\mathbb{E}[(v^{k,r}_t - \pi^{k,r})^2] \notag \\
    =& \mathbb{E}[(\frac{1}{t}\sum_{\tau=1}^t u^{k,r}_\tau - \pi^{k,r})^2] = \mathbb{E}[(\frac{1}{t}\sum_{\tau=1}^t (u^{k,r}_\tau - \pi^{k,r}))^2] \\
   =&\frac{1}{t^2}\sum_{\tau=1}^t \mathbb{E}\left[(u^{k,r}_\tau - \pi^{k,r}) \sum_{\tau'=1}^t (u^{k,r}_{\tau'} - \pi^{k,r})\right] \leq  \frac{2\Gamma + 1}{t}\delta^2
\end{align}
where the last inequality uses Assumption \ref{assm1}.

\section{Proof of Lemma \ref{lemma1}}
\label{appexlemma1}
The difference between the \textit{real local gradient} and \textit{virtual local gradient} can be bounded as follows:
\begin{align}
    & \mathbb{E} [| g_{t, \tau}^k - \hat{g}_{t, \tau}^k |^2] \\
    & = \mathbb{E} [|\sum_{r=1}^R {v}^{k, r}_t \nabla F^k(w^k_{t, \tau}; \mathcal{L}^{k,r}_t) - \sum_{r=1}^R  {\pi}^{k, r}\nabla F^k(w^k_{t, \tau}; \mathcal{L}^{k,r}_t)|^2]
     \\ & = \mathbb{E} [|\sum_{r=1}^R ({v}^{k, r}_t-{\pi}^{k, r}) \nabla F^k(w^k_{t, \tau}; \mathcal{L}^{k,r}_t)|^2]
     \\ & \leq R  \sum_{r=1}^R \mathbb{E} [| ({v}^{k, r}_t-{\pi}^{k, r}) \nabla F^k(w^k_{t, \tau}; \mathcal{L}^{k,r}_t)|^2]
     \\ & \leq  R  \sum_{r=1}^R \mathbb{E} [| ({v}^{k, r}_t-{\pi}^{k, r})|^2] \mathbb{E} [|\nabla F^k(w^k_{t, \tau}; \mathcal{L}^{k,r}_t)|^2]
     \\ & \leq R^2 \lambda^2_t \sigma_M^2
\end{align}
Then the proof of the second inequality is as follows:
\begin{align}
 & \mathbb{E} [| \hat{g}_{t, \tau}^k - \nabla f^{k}(w^k_{t,\tau})|^2] \\
 & = \mathbb{E} [|\sum_{r=1}^R {\pi}^{k, r} \nabla F^k(w^k_{t, \tau}; \mathcal{L}^{k,r}_t) - \sum_{r=1}^R  {\pi}^{k, r}\nabla f^{k,r}(w^k_{t, \tau})|^2]
 \\& = \mathbb{E} [|\sum_{r=1}^R {\pi}^{k, r}( \nabla F^k(w^k_{t, \tau}; \mathcal{L}^{k,r}_t) -\nabla f^{k,r}(w^k_{t, \tau}))|^2]
     \\ & \leq R \overline{\pi}^2 \sum_{r=1}^R \mathbb{E} [|( \nabla F^k(w^k_{t, \tau}; \mathcal{L}^{k,r}_t) -\nabla f^{k,r}(w^k_{t, \tau}))|^2]
 \\&\leq 2 R^2 \overline{\pi}^2 \sigma_M^2
\end{align}
where $\overline{\pi} = \max_{k,r} \pi^{k, r}$ is the maximum ratio of the long-term label distribution.

\section{Proof of Lemma \ref{lemma2}}
\label{appexlemma2}
In this subsection, we will get the local updates bound,

\begin{align}
 &\mathbb{E}[\|{w}^k_{t, \tau} - {w}_t\|^2]\\
 = & \mathbb{E}[\|{w}^k_{t, \tau-1 } - {w}_t- \eta_L {g}_{t, \tau - 1}^{k}\|^2 ]\\
 = & \mathbb{E}[\|{w}^k_{t, \tau-1 } - {w}_t- \eta_L ({g}_{t, \tau - 1}^{k} - \hat{g}_{t, \tau - 1}^{k} + \hat{g}_{t, \tau - 1}^{k}  \notag\\
- & \nabla f^{k}({w}^{k}_{t,\tau - 1}) + \nabla f^{k}({w}^{k}_{t,\tau - 1}) -  \nabla f^{k}({w}_{t}) + \nabla f^{k}({w}_{t}) ) \|^2 ]\\
\leq &\left(1 + \frac{1}{2E -1} \right) \mathbb{E}\|{w}^k_{t, \tau-1} - {w}_t\|^2 + \eta_L^2 \mathbb{E}\|{g}_{t, \tau - 1}^{k} - \hat{g}_{t, \tau - 1}^{k}\|^2 \nonumber\\
+ & 6 E\eta_L^2 \mathbb{E}\|\hat{g}_{t, \tau - 1}^{k} - \nabla f^{k}({w}^{k}_{t,\tau - 1})\|^2 \nonumber \\
+ & 6 E \eta_L^2 \mathbb{E} \| \nabla f^{k}({w}^{k}_{t,\tau - 1}) -  \nabla f^{k}({w}_{t})\|^2 + 6 E \eta_L^2 \mathbb{E} \| \nabla f^{k}({w}_{t})\|^2\\
\leq & \left(1 + \frac{1}{2E -1} \right) \mathbb{E}\|{w}^k_{t, \tau-1} - {w}_t\|^2 +  \eta_L^2 R^2 \lambda^2_t \sigma_M^2  \nonumber \\
+ & 12 E\eta_L^2 R^2 \overline{\pi}^2 \sigma_M^2  +  6E \eta_L^2 L^2 \mathbb{E} \| {w}^{k}_{t,\tau - 1} -  {w}_{t}\|^2 \nonumber\\
+  & 6E \eta_L^2 \sigma^2_G + 6E \eta_L^2 (A^2 +1) \|\nabla f(x) \|^2 \\
\leq & \left(1 + \frac{1}{E -1} \right)\mathbb{E}\|{w}^k_{t, \tau-1} - {w}_t\|^2 + \eta_L^2 R^2 \lambda^2_t \sigma_M^2 \nonumber\\
+ &  12 E\eta_L^2 R^2 \overline{\pi}^2 \sigma_M^2 +  6E \eta_L^2 \sigma^2_G + 6E \eta_L^2 (A^2 +1) \|\nabla f(x) \|^2
\end{align}

Unrolling the recursion, we get:
\begin{align}
& \frac{1}{K}\sum_{k=1}^K \mathbb{E}[\|{w}^k_{t, \tau} - {w}_t\|^2]\\ 
\leq & \sum_{p=0}^{\tau-1} \left(1 + \frac{1}{E -1} \right)^p [ \eta_L^2 R^2 \lambda^2_t \sigma_M^2 + 12 E\eta_L^2 R^2 \overline{\pi}^2 \sigma_M^2  \nonumber \\
+ & 6E \eta_L^2 \sigma^2_G + 6E \eta_L^2 (A^2 +1) \|\nabla f(x) \|^2  ] \\
\leq & (E-1) \left [(1 +   \frac{1}{E -1})^E - 1\right ]  [ \eta_L^2 R^2 \lambda^2_t \sigma_M^2 \nonumber \\
+ & 12 E\eta_L^2 R^2 \overline{\pi}^2 \sigma_M^2 +  6E \eta_L^2 \sigma^2_G + 6E \eta_L^2 (A^2 +1) \|\nabla f(x) \|^2  ]\\
 \leq & 5 E \eta_L^2 R^2 \lambda^2_t \sigma_M^2  +  60E^2\eta_L^2 R^2 \overline{\pi}^2 \sigma_M^2 \nonumber \\
+  & 30E^2 \eta_L^2 \sigma^2_G + 30 E^2 \eta_L^2 (A^2 +1) \|\nabla f(x) \|^2
\end{align}
This completes the proof of lemma \ref{lemma2}.

\section{Proof of Theorem \ref{thm1}}
\label{appdex:theorem1}
In this section, we give the proofs in detail. Due to the smoothness in Assumption \eqref{assm:smoothness}, taking expectation of $f(w_{t+1})$ over the randomness in round $t$, we have
\begin{align}
    & \mathbb{E}_t[f(w_{t+1})]  \\
    \leq & f(w_t) + \langle \nabla f(w_t), \mathbb{E}_t[w_{t+1} - w_t]\rangle + \frac{L}{2}\mathbb{E}_t[\|w_{t+1} - w_t\|^2]\\
    = & f(w_t) + \langle \nabla f(w_t), \mathbb{E}_t[\eta \eta_L \Delta_t + \eta\eta_L E \nabla f(w_t) 
     \notag \\
    - & \eta\eta_L E \nabla f(w_t)]\rangle + \frac{L}{2}\eta^2\eta_L^2 \mathbb{E}_t[\|\Delta_t\|^2]\\
    = & f(w_t) - \eta\eta_L E\|\nabla f(w_t)\|^2 \notag \\
    + & \eta \underbrace{ \langle \nabla f(w_t), \mathbb{E}[ \eta_L \Delta_t + \eta_L E\nabla f(w_t)]\rangle}_{A_1} 
     + \frac{L}{2}\eta^2 {\eta_L^2}\underbrace{\mathbb{E}_t[\|\Delta_t\|^2]}_{A_2}
\end{align}
Note that the term $A_1$ can be bounded as follows: 
\begin{align}
    & A_1 = \langle {\nabla f(w_t)}, \mathbb{E}_t[\eta_L {\Delta_t} + \eta_L E {\nabla f(w_t)}]\rangle\\
    = & \langle {\nabla f(w_t)}, \mathbb{E}_t[\eta_L {\bar{\Delta}_t} + {\eta_L e_t} + \eta_L E {\nabla f(w_t)}]\rangle\\
    = &\langle {\nabla f(w_t)}, \mathbb{E}_t [-\frac{1}{K}\sum_{k=1}^K\sum_{\tau = 0}^{E-1} \eta_L  {\nabla {F}^k({w}^k_{t,\tau})} \notag\\
    + & {\eta_L e_t} + \eta_L E\frac{1}{K}\sum_{k=1}^K {\nabla F^k(w_t)} ] \rangle\\
    = & \langle \sqrt{\eta_L E} \nabla f(w_t), -\frac{\sqrt{\eta_L}}{K\sqrt{E}}\mathbb{E}_t[\sum_{k=1}^K\sum_{\tau = 0}^{E-1}(\nabla {F}^k({w}^k_{t, \tau}) \notag\\
    - & \nabla F^k(w_t)) - {K e_t} ]\rangle\\
    \overset{(a_1)}{=} & \frac{\eta_L E}{2}\|\nabla f(w_t)\|^2   \nonumber\\
    + &\frac{\eta_L}{2E K^2} \mathbb{E}_t \left\|\sum_{k=1}^K\sum_{\tau = 0}^{E-1}(\nabla {F}^k({w}^k_{t, \tau}) - \nabla F^k(w_t)) - {K e_t} \right\|^2  \nonumber\\
    - &\frac{\eta_L}{2E K^2}\mathbb{E}_t\left\| \sum_{k=1}^K\sum_{\tau = 0}^{E-1}\nabla {F}^k({w}^k_{t, \tau})  - {K e_t} \right\|^2\\
    \overset{(a_2)}{\leq}& \frac{\eta_L E}{2}\|\nabla f(w_t)\|^2 + \frac{\eta_L}{E K^2} \mathbb{E}_t \left\|\sum_{k=1}^K\sum_{\tau = 0}^{E-1}(\nabla {F}^k({w}^k_{t, \tau}) - \nabla F^k(w_t))\right\|^2  \nonumber\\
    - &\frac{\eta_L}{2E K^2}\mathbb{E}_t\left\| \sum_{k=1}^K\sum_{\tau = 0}^{E-1}\nabla {F}^k({w}^k_{t, \tau})  - {K e_t}\right\|^2 + {\frac{\eta_L \mathbb{E}_t\|e_t\|^2}{E}}\\
    \overset{(a_3)}{\leq}& \frac{\eta_L E}{2}\|\nabla f(w_t)\|^2 + \frac{\eta_L}{K} \sum_{k=1}^K\sum_{\tau = 0}^{E-1} \mathbb{E}_t \left\|{\nabla {F}^k({w}^k_{t, \tau})} - {\nabla F^k(w_t)}\right\|^2  \nonumber\\
    - &\frac{\eta_L}{2E K^2}\mathbb{E}_t\left\| \sum_{k=1}^K\sum_{\tau = 0}^{E-1}\nabla {F}^k({w}^k_{t, \tau})  - {K e_t} \right\|^2 + {\frac{\eta_L \mathbb{E}_t\|e_t\|^2}{E}}\\
    \overset{(a_4)}{\leq}& \frac{\eta_L E}{2}\|\nabla f(w_t)\|^2 + \frac{\eta_L L^2}{K} \sum_{k=1}^K\sum_{\tau = 0}^{E-1} \mathbb{E}_t \left\|{{w}^k_{t, \tau}}- {w_t}\right\|^2  \nonumber\\
    - &\frac{\eta_L}{2E K^2}\mathbb{E}_t\left\| \sum_{k=1}^K\sum_{\tau = 0}^{E-1}\nabla F^k({w}^k_{t, \tau})  - {K e_t}\right\|^2 + {\frac{\eta_L \mathbb{E}_t\|e_t\|^2}{E}}\\
    \overset{(a_5)}{\leq}& \eta_L E(\frac{1}{2} + 30(A^2 +1)\eta^2_L E^2L^2)\|\nabla f(w_t)\|^2  \nonumber\\
    + & 5 \eta^3_L E L^2 \left ( R^2 \lambda^2_t \sigma_M^2  +  12 E R^2 \overline{\pi}^2 \sigma_M^2 + 6 E \sigma^2_G\right )  \nonumber\\
    - & \frac{\eta_L}{2E K^2}\mathbb{E}_t\left\| \sum_{k=1}^K\sum_{\tau = 0}^{E-1}\nabla F^k({w}^k_{t, \tau})  - {K e_t}\right\|^2 + {\frac{\eta_L \mathbb{E}_t\|e_t\|^2}{E}} 
\end{align}
where $(a_1)$ follows from that $\left\langle \x, \y \right\rangle = \frac{1}{2}[\|\x\|^2 + \|\y\|^2 - \|\x - \y\|^2]$, $(a_2)$ is due to that $\mathbb{E}\|x_1 + x_2\|^2 \leq 2\mathbb{E}[\|x_1\|^2 + \|x_2\|^2]$, $(a_3)$ is due to that $\mathbb{E}\|x_1 + ... + x_n\|^2 \leq n\mathbb{E}[\|x_1\|^2 + ... \|x_n\|^2]$, $(a_4)$ is due to Assumption \eqref{assm:smoothness} and $(a_5)$ follows from Lemma \ref{lemma1}. 

The term $A_2$ can be bounded as
\begin{align}
    A_2 =& \mathbb{E}_t[\| {\Delta_t}\|^2]   = \mathbb{E}_t[\| {\bar{\Delta}_t} +  e_t\|^2]\\
    \overset{(a_6)}{\leq}& 2 \mathbb{E}_t\|\bar{\Delta}_t\|^2 + 2 \mathbb{E}_t\|e_t\|^2  \\
    \leq & \frac{2}{K^2}\mathbb{E}_t\left[\left\|\sum_{k=1}^K\sum_{\tau=0}^{E-1} \hat{g}^k_{t,\tau}\right\|^2\right] + 2 \mathbb{E}_t\|e_t\|^2\\
    \overset{(a_7)}{\leq}& \frac{2}{K^2}\mathbb{E}_t\left[\left\|\sum_{k=1}^K\sum_{\tau=0}^{E-1} (\hat{g}^k_{t,\tau} - \nabla F^k({w}^k_{t, \tau}))\right\|^2\right] \nonumber\\
     + & \frac{2}{K^2}\mathbb{E}_t\left[\left\|\sum_{k=1}^K\sum_{\tau=0}^{E-1} \nabla F^k({w}^k_{t, \tau})\right\|^2\right]+  2  \mathbb{E}_t\|e_t\|^2\\
    \overset{(a_8)}{\leq} &\frac{4E}{K} R^2\overline{\pi}^2 \sigma_M^2 + \frac{4}{K^2}\mathbb{E}_t\left[\left\|\sum_{k=1}^K\sum_{\tau=0}^{E-1} \nabla F^k({w}^k_{t, \tau}) - {K e_t}\right\|^2\right] \nonumber\\
    + & \frac{4}{K^2}\mathbb{E}_t\left\|{K e_t}\right\|^2 +  2  \mathbb{E}_t\|e_t\|^2 \\
    \leq & \frac{4E}{K} R^2 \overline{\pi}^2 \sigma_M^2 \nonumber\\
    + & \frac{4}{K^2}\mathbb{E}_t\left[\left\|\sum_{k=1}^K\sum_{\tau=0}^{E-1} \nabla F^k({w}^k_{t, \tau}) -{K e_t}\right\|^2\right] +  6  \mathbb{E}_t\|e_t\|^2
\end{align}
where both $(a_6)$ is due to that $\mathbb{E}\|x_1 + x_2\|^2 \leq 2\mathbb{E}[\|x_1\|^2 + \|x_2\|^2]$, $(a_7)$ follows the fact that $\mathbb{E}[\|\x\|^2] = \mathbb{E}[\|\x - \mathbb{E}\x\|^2] + \|\mathbb{E}\x\|^2$, and $(a_8)$ is due to Assumption \eqref{assm:unbiased-local}

Substituting the inequalities of $A_1$ and $A_2$ into the original inequality, we have:
\begin{align}
&\mathbb{E}_t[f(w_{t+1})] \\
& \leq f(w_t) - \eta\eta_L E\|\nabla f(w_t)\|^2 \nonumber\\
& + \eta \underbrace{ \langle \nabla f(w_t), \mathbb{E}[\eta_L\Delta_t + \eta_L E\nabla f(w_t)]\rangle}_{A_1} +\frac{L}{2}\eta^2 \eta_L^2 \underbrace{\mathbb{E}_t[\|\Delta_t\|^2]}_{A_2}\\
&\leq  f(w_t) - \eta\eta_L E\|\nabla f(w_t)\|^2 \nonumber\\
&+\eta\eta_L E(\frac{1}{2} + 30(A^2 +1)\eta^2_L E^2L^2)\|\nabla f(w_t)\|^2 \nonumber\\
&+ 5\eta \eta^3_LE^2 L^2 \left ( R^2 \lambda^2_t \sigma_M^2  +  12 E R^2 \overline{\pi}^2 \sigma_M^2 + 6 E \sigma^2_G\right )  \nonumber\\
&- \frac{\eta\eta_L}{2E K^2}\mathbb{E}_t\left\| \sum_{k=1}^K\sum_{\tau = 0}^{E-1}\nabla F^k({w}^k_{t, \tau})  -{K e_t}\right\|^2 \nonumber\\
& + {\frac{\eta\eta_L \mathbb{E}_t\|e_t\|^2}{E}} + \frac{2 E L\eta^2\eta^2_L}{K} R^2 \overline{\pi}^2 \sigma_M^2\nonumber\\
&+ \frac{2L\eta^2\eta^2_L}{K^2}\mathbb{E}_t\left[\left\|\sum_{k=1}^K\sum_{\tau=0}^{E-1} \nabla F^k({w}^k_{t, \tau}) - {K e_t}\right\|^2\right] +  3\eta^2\eta^2_L L \mathbb{E}_t\|e_t\|^2\\
&= f(w_t) - \eta\eta_L E(\frac{1}{2} - 30(A^2 +1)\eta^2_L E^2L^2)\|\nabla f(w_t)\|^2  \nonumber\\
&+ 5\eta \eta^3_LE^2 L^2 R^2 \lambda^2_t \sigma_M^2 + 60\eta \eta^3_LE^3 L^2 R^2 \overline{\pi}^2 \sigma_M^2 +  30\eta \eta^3_L E^3 L^2 \sigma^2_G \nonumber\\
&+ {\left(\frac{\eta\eta_L}{E} + 3\eta^2 \eta^2_L L\right)} \mathbb{E}_t \left \|\frac{1}{K} \sum_{k=1}^K\sum_{\tau = 0}^{E-1}\left ( \hat{g}_{t, \tau}^k - g_{t,\tau}^k \right )\right \|^2 \nonumber\\
&- \left(\frac{\eta\eta_L}{2E K^2} - \frac{2L\eta^2\eta^2_L}{K^2}\right) \mathbb{E}_t\left\| \sum_{k=1}^K\sum_{\tau = 0}^{E-1}\nabla F^k({w}^k_{t, \tau})  - {K e_t}\right\|^2 \\
&\overset{(a_9)}{\leq} f(w_t) - c\eta\eta_L E\|\nabla f(w_t)\|^2+ 60\eta \eta^3_LE^3 L^2 R^2 \overline{\pi}^2\sigma_M^2 \nonumber \\
& +  30\eta \eta^3_L E^3 L^2 \sigma^2_G + \left (5\eta \eta^3_LE^2 L^2 +  \eta \eta_L E + 3 \eta^2 \eta_L^2 L E^2 \right ) R^2 \lambda^2_t \sigma_M^2
\end{align}
where $(a_9)$ follows from $\left(\frac{\eta\eta_L}{2E K^2} - \frac{2L \eta^2\eta^2_L}{K^2}\right) > 0$ if $\eta\eta_L \leq \frac{1}{4E L}$, and that there exits a constant $c > 0$ satisfying $(\frac{1}{2} - 30(A^2 +1)\eta^2_L E^2L^2) > c > 0$ if $\eta_L < \frac{1}{\sqrt{60 (A^2+1) }EL}$. 

Rearranging and summing from $t = 0, ..., T -1$, we have:
\begin{align}
    &\sum_{t=0}^{T-1} c\eta\eta_L E\mathbb{E}\|\nabla f(w_t)\|^2 \\
    \leq& f(w_0) - f(w_T) + E\eta\eta_L\sum_{t=0}^{T-1}60 \eta^2_L E^2 L^2 R^2 \overline{\pi}^2 \sigma_M^2 \nonumber \\
    +& E\eta\eta_L\sum_{t=0}^{T-1} 30 \eta^2_L E^2 L^2 \sigma^2_G \nonumber\\
    +& E\eta\eta_L\sum_{t=0}^{T-1}  \left (5 \eta^2_L E L^2 + 3 \eta \eta_L L E + 1  \right ) R^2 \lambda^2_t \sigma_M^2
\end{align}
which implies,
\begin{align}
    \min_{t = 0,..., T-1}\mathbb{E}\|\nabla f(w_t)\|^2 \leq \frac{f_0 - f_*}{c\eta\eta_L E T} + \Phi_G + \Phi_M + \Phi_L
\end{align}
where 
\begin{align}
    &\Phi_G = \frac{30 E^2 \eta_L^2 L^2}{c}\sigma_G^2\\
    &\Phi_M = \frac{60 \eta^2_L E^2 L^2 R^2 \overline{\pi}^2 }{c}\sigma^2_M \\
    &\Phi_L = \frac{\left (5 \eta^2_L E L^2 + 3 \eta \eta_L L  E + 1\right ) R^2 \sigma_M^2}{c T} \sum_{t=0}^{T-1}\lambda^2_t
\end{align}
This completes the proof.
\end{document}